\documentclass[twoside]{article}

\usepackage[accepted]{aistats2026}

\usepackage[dvipsnames]{xcolor}

\usepackage{amsmath, amssymb, amsthm,mathtools,dsfont}
\usepackage{bbm}
\usepackage{blkarray}
\usepackage{cancel}
\usepackage{enumitem}
\usepackage{euscript}
\usepackage{float}
\usepackage{bm}
\usepackage{graphicx}
\usepackage{mathrsfs}
\usepackage{microtype}
\usepackage{ragged2e}
\usepackage{sectsty}
\usepackage{thmtools}
\usepackage{tikz}
\usepackage[Symbolsmallscale]{upgreek}
\usepackage{multirow}
\usepackage{booktabs}

\usepackage{microtype}
\usepackage{graphicx}
\usepackage{subfigure}
\usepackage{hyperref}

\usepackage{caption}
\usepackage{subcaption}

\hypersetup{citecolor=purple, colorlinks=true, linkcolor=blue}

\newcommand{\cC}{\mathcal{C}}

\newcommand{\cF}{\mathcal{F}}

\newcommand{\cH}{\mathcal{H}}
\newcommand{\cI}{\mathcal{I}}

\newcommand{\cN}{\mathcal{N}}
\newcommand{\cO}{\mathcal{O}}
\newcommand{\cP}{\mathcal{P}}

\newcommand{\cT}{\mathcal{T}}

\newcommand{\cV}{\mathcal{V}}

\newcommand{\PP}{\mathbb{P}}

\newcommand{\WW}{\mathbb{W}}

\newcommand{\kl}[2]{\KL(#1 \!\;\|\; \!#2)}

\newcommand*{\norm}[1]{\left\|#1\right\|}
\newcommand*{\triplenorm}[1]{{\left\vert\kern-0.25ex\left\vert\kern-0.25ex\left\vert #1
    \right\vert\kern-0.25ex\right\vert\kern-0.25ex\right\vert}}

\DeclareMathOperator{\id}{id}

\newcommand{\R}{\mathbb{R}}

\renewcommand{\phi}{\varphi}
\newcommand{\eps}{\varepsilon}
\renewcommand{\epsilon}{\varepsilon}

\newcommand*{\E}{\mathbb E}

\DeclareMathOperator{\tr}{tr}

\newcommand*{\defeq}{\coloneqq}
\newcommand*{\rd}{\mathrm{d}}
\newcommand*{\dd}{\, \rd}
\DeclareMathOperator*{\argmin}{arg\,min}

\DeclareMathOperator*{\KL}{KL}

\newcommand{\pirad}{\pi_{\rm rad}}
\newcommand{\Crad}{\cC_{\rm rad}}
\newcommand{\Unif}{\mathsf{Unif}}

\usepackage{algorithm}
\usepackage{algpseudocode}

\usepackage{fancyhdr}

\usepackage[capitalize]{cleveref}

\usepackage[textsize=tiny]{todonotes}

\theoremstyle{plain}
\newtheorem{theorem}{Theorem}[section]

\newtheorem{prop}[theorem]{Proposition}
\newtheorem{lemma}[theorem]{Lemma}

\theoremstyle{definition}

\theoremstyle{remark}
\newtheorem{remark}[theorem]{Remark}

\usepackage[dvipsnames]{xcolor}

\usepackage[round]{natbib}

\renewcommand{\tilde}{\widetilde}

\begin{document}

\runningtitle{Variational inference via radial transport}

\twocolumn[

\aistatstitle{Variational inference via radial transport}

\aistatsauthor{ Luca Ghafourpour$^{1,2}$ \qquad Sinho Chewi$^{3}$ \qquad Alessio Figalli$^{1}$ \qquad  Aram-Alexandre Pooladian$^{3}$  }
\vspace{2mm}
\aistatsaddress{ $^{1}$ETH Z\"urich \qquad $^{2}$Cambridge \qquad  $^{3}$Yale University}
\vspace{-6mm}
\aistatsaddress{ \texttt{ldg34@cam.ac.uk} \qquad \texttt{afigalli@ethz.ch} \qquad  \texttt{\{sinho.chewi,aram-alexandre.pooladian\}@yale.edu}}]
\begin{abstract}
In variational inference (VI), the practitioner approximates a high-dimensional distribution $\pi$ with a simple surrogate one, often a (product) Gaussian distribution. However, in many cases of practical interest, Gaussian distributions might not capture the correct radial profile of $\pi$, resulting in poor coverage. In this work, we approach the VI problem from the perspective of optimizing over these radial profiles. Our algorithm \texttt{radVI} is a cheap, effective add-on to many existing VI schemes, such as Gaussian (mean-field) VI and Laplace approximation. We provide theoretical convergence guarantees for our algorithm, owing to recent developments in optimization over the Wasserstein space---the space of probability distributions endowed with the Wasserstein distance---and new regularity properties of radial transport maps in the style of \citet{Caffarelli2000}. 
\end{abstract}

\section{Introduction}
Variational inference (VI) is a fundamental optimization problem that takes place over subsets of probability distributions \citep{Wainwright2008, Blei2017}. We consider a standard setup that arises in many applications, where the practitioner is given a high-dimensional posterior distribution $\pi \propto \exp(-V)$ and the goal is to solve
\begin{align*}
    \pi^\star_{\cC} \defeq \argmin_{\mu \in \cC} \kl{\mu}{\pi}\,,
\end{align*}
where $\cC \subset \cP(\R^d)$ is a fixed set of probability distributions. VI is a powerful computational stand-in for standard Markov Chain Monte Carlo (MCMC) methods for sampling from unnormalized posteriors $\pi$. Indeed, while MCMC methods require simulating Markov chains for prohibitively long periods of time, it might be possible to instead quickly learn a surrogate density that is a good enough approximation to the posterior for practical purposes; see the review by \citet{Blei2017} for more details.

In VI, the choice of $\cC \subset \cP(\R^d)$ is of the utmost importance. For example, the case where $\cC$ is the set of all Gaussians (with positive definite covariance) is known as \emph{Gaussian VI} \citep{BarBis1997Ensemble, See1999BayesianModel, OppArc09VarGauss}. In large-scale machine learning applications, it is also common to optimize over the class of Gaussians with diagonal covariance, resulting in mean-field Gaussian VI. While these algorithms have a long history, a rigorous, theoretical analysis is only just emerging, based on the theory of optimal transport through Wasserstein gradient flows \citep{ambrosio2008gradient}. For example, the Gaussian case has been studied by \cite{Lambert2022,Diao2023, kim2023convergence}. We note that it is possible to implement algorithms based on mixtures of Gaussians as outlined by \citet{Lambert2022,petit2025variational}, though the mathematical analysis in this case is significantly more challenging. Separately, \emph{Laplace approximation} is an alternative means of obtaining a surrogate measure to $\pi$, where one considers the following Gaussian approximation $ \cN(x^\star,(\nabla^2V(x^\star))^{-1})$
where $x^\star = \text{argmin} \, \, V$. The literature on Laplace approximations is vast; see e.g.,~\citet{Robert2004}.

\cite{margossian2025variational} highlight the strengths and weaknesses of existing VI-based algorithms. Notably, they provide some characterizations for when VI can hope to exactly recover the mean and correlation matrix of a target distribution $\pi$. While they are only interested in these particular statistics, a key detail of the paper is that the variational approximating family must be decided in advance, which leads to demonstrable shortcomings even in small-scale examples.

\subsection{Contributions}
To mitigate the issues brought about by these approximations, we study the VI problem over \emph{radial profiles}. For fixed $m \in \R^d$ and $\Sigma \succ 0$, we consider the following variational family:
\begin{align*}
    p_h(x) \propto \det(\Sigma)^{-1/2}\,h((x-m)^\top\Sigma^{-1}(x-m))\,,
\end{align*}
{as $h$ ranges over non-negative functions on $[0,\infty)$.
If $m$ and $\Sigma$ are known, or if estimates thereof can be imputed, then we can assume $m=0$ and $\Sigma=I$ via a whitening procedure (see Section~\ref{ssec:whiten}). We henceforth assume that this has been done, so that our variational family is the set $\cC_{\rm rad}$ of \emph{radially symmetric} distributions.} This family encompasses the standard Gaussian via $h(y) = \exp(-y/2)$, but also Student-$t$ distributions, the non-smooth Laplace distribution, the logistic distribution, among others.\footnote{On the other hand, if we fix $h$ and vary $(m,\Sigma)$, we end up at the theory of elliptical families, see e.g.,~\citet[Section 1.4]{Mui1982Multivariate}.}

In this paper, we propose and analyze a tractable algorithm for solving
\begin{align*}
    \pi_{\rm rad}^\star \defeq  \argmin_{\mu \in \cC_{\rm rad}}\kl{\mu}{\pi}\,,
\end{align*}
where $\pi \propto \exp(-V)$. Our contributions are of both theoretical and computational interest. We stress that our only assumptions throughout this work will be on the true posterior $\pi$, namely that $\pi$ is log-smooth and strongly log-concave {and centered at the origin}. This pair of assumptions has been leveraged in nearly all works on the theory of sampling~\citep{ChewiBook} and in the theoretical and computational study of variational inference~\citep{Lambert2022, arnese2024convergence, Lacker2024, lavenant2024convergence, JiaChePoo24MFVI}.

In Section~\ref{sec:setup_theory}, we prove existence and uniqueness of the radial minimizer $\pirad^\star$, as well as establish regularity properties of said minimizer. For example, if $\pi$ is log-smooth and strongly log-concave, {Proposition}~\ref{reg_radial} states that $\pirad^\star$ is as well. We also prove Caffarelli-type contraction estimates \citep{Caffarelli2000} for the corresponding optimal radial transport map $T_{\rm rad}^\star$ from the standard Gaussian $\rho = \cN(0,I)$, say, to $\pirad^\star$; see Theorem~\ref{theorem_map_regularity}.

In Section~\ref{sec:comp_guarantees}, we embrace the conventional wisdom of ``parametrizing then optimizing'' in order to compute $\pirad^\star$, leading to our proposed algorithm \texttt{radVI} (see Algorithm~\ref{alg:radVI}). Concretely, we make use of the representation of a given radial measure as the pushforward of the standard Gaussian by a radial map $T_{\rm rad}$. Our approach is based on carefully parametrizing radial transport maps $T_\lambda$ for $\lambda \in \R^{J+1}_+$ for some $J > 0$ where, if $J$ is large enough, our parameterized set should encompass all possible radial maps (see Theorem~\ref{theorem_unif_approx}). Then, writing our objective over the non-negative orthant as

\begin{align*}
    \lambda \mapsto \cF(\lambda) = \kl{(T_\lambda)_\sharp\rho}{\pi} \quad (\lambda\in\R^{J+1}_+)
\end{align*}
we show that standard Euclidean gradient descent converges to the \emph{true} optimal radial transport map $T^\star_{\rm rad}$, i.e., $ \|T_{\lambda^{(k)}} - T^\star_{\rm rad}\|_{L^2(\rho)} \leq \epsilon$ for all $k \ge \widetilde\Omega(\epsilon^{-1})$ up to polynomial factors of the condition number of $\pi$---see Theorem~\ref{thm:radvi_converges}.  {Notably, our convergence guarantees are effectively \emph{dimension-free}}.

In Section~\ref{sec:experiments}, we {show how \texttt{radVI} can be (easily) used to improve existing VI methods on a collection of synthetic examples}.  In addition to recovering isotropic profiles, we also show how \texttt{radVI} can be used as a preconditioner: given any method of obtaining a mean-covariance proxy, such as through Laplace approximation or Gaussian VI, we show how \texttt{radVI} can often improve upon the existing approximation {by better capturing the tail behavior of the posterior}. {As a final example, we turn to parameter estimation problems (estimating the second moment, or probability thresholds) given an unnormalized posterior. Focusing on the Neal's funnel distribution, we show how \texttt{radVI} on top of full-rank Gaussian VI can lead to substantially improved estimates of these quantities at virtually no added computational cost. We believe these findings merit further investigation.} 
\subsection*{Notation}
We write $\mu \in \cP(\R^d)$ if $\mu$ is a probability measure over $\R^d$; $\mu \in \cP_2(\R^d)$ if $\mu$ has finite second moment; $\mu \in \cP_{\rm ac}(\R^d)$ if $\mu$ has a Lebesgue density; and $\cP_{2,\rm ac}(\R^d) =  \cP_{\rm ac}(\R^d) \cap \cP_{2}(\R^d)$. For positive constants $a$ and $b$, we write $a \lesssim b$ or $a = \cO(b)$ (resp.\ $a \gtrsim b$ or $a = \Omega(b)$) to mean there exists a positive constant $C$ such that $a \leq C b$ (resp.\ $a \geq Cb$). If both $a \lesssim b$ and $b \lesssim a$, we write $a \asymp b$. When omitting logarithmic factors in $b$, we write, e.g., $a \lesssim_{\log} b$ or $a = \tilde\cO(b)$ (analogously for lower bounds). Throughout, implied constants will always be independent of the dimension and other relevant problem parameters. We write the uniform distribution on the unit sphere (denoted $S^{d-1}$) in $\R^d$ as $\Unif$. 
Recall that a function $f:\R^d\to\R$ is $M$-smooth and $m$-strongly convex if $0 \preceq mI \preceq \nabla^2 f \preceq MI$.
\section{Background}
\subsection{Primer on optimal transport}
For $\mu, \nu \in \cP_{2}(\R^d)$, the squared \emph{$2$-Wasserstein distance} between them is defined as
\begin{align}\label{def:wass}
    W_2^2(\mu,\nu) \defeq \inf_{\pi \in \Pi(\mu,\nu)} \iint \|x-y\|^2 \dd \pi(x,y)\,,
\end{align}
where $\Pi(\mu,\nu)$ is the set of joint measures with first marginal $\mu$ and second marginal equal to $\nu$. For more details, see \cite{villani2009optimal,santambrogio2015optimal, CheNilRig25OT}.

The associated \emph{optimal transport map} between $\mu$ and $\nu$ is given by
\begin{align}\label{def:otmap}
    T^\star \defeq \argmin_{T \in \cT(\mu,\nu)} \int \|x - T(x)\|^2 \dd \mu(x)\,,
\end{align}
where $\cT(\mu,\nu)$ is the set of admissible transport maps between $\mu$ and $\nu$, which consist of vector-valued functions $T :\R^d \to \R^d$ such that for $X\sim\mu$, it holds that $T(X)\sim\nu$. Without further assumptions, it is possible that such a $T^\star$ need not exist. A theorem due to \citet{brenier1991polar} unifies both \eqref{def:wass} and \eqref{def:otmap}.
\begin{theorem}[Brenier's theorem]
Suppose $\mu \in \cP_{2,\rm{ac}}(\R^d)$. Then \eqref{def:otmap} has a unique minimizer, with
\begin{align*}
    W_2^2(\mu,\nu) = \int \|x - T^\star(x)\|^2 \dd \mu(x)\,,
\end{align*}    
and $T^\star = \nabla \phi^\star$ for some convex function $\phi^\star$. 
\end{theorem}

\subsection{Geodesic convexity via compatible families of optimal transport maps}
We now consider the space of probability measures with densities endowed with the $2$-Wasserstein distance, called the \emph{Wasserstein space} $\WW_2 \defeq (\cP_{2,\rm{ac}}(\R^d), W_2)$.
Our goal is to formulate the VI problem as an optimization problem over this space.

We let $\cT_{\rm rad}$ denote the set of radial (transport) maps, which are vector-valued functions $T_{\rm rad}:\R^d\to\R^d$ with
\begin{align}
    x \mapsto T_{\rm rad}(x) = {\Psi(\|x\|)\,x/\|x\|}\,,
\end{align}
where $\Psi : [0,\infty) \to [0,\infty) $ is strictly monotone and continuous. One can check that $\cT_{\rm rad}$ forms a compatible family of optimal transport maps~\citep[see][Section~2.3.2]{panaretos2020invitation}, and thus $(\cT_{\rm rad})_\sharp\rho$ forms a geodesically convex subset of the Wasserstein space. 

For an introduction to the space $\WW_2$ focused on applications to statistics, see \cite{panaretos2020invitation,CheNilRig25OT}.

\section{Setup and theoretical results}\label{sec:setup_theory}
Given a posterior $\pi \propto \exp(-V)$, our objective is to compute the following minimizer:
\begin{align}\label{eq:radvi}\tag{$\mathsf{radVI}$}
    \pi_{\rm rad} \in \argmin_{\mu \in \Crad} \kl{\mu}{\pi}\,,
\end{align}
where we recall that $\Crad$ is the set of radially symmetric measures. In this section, we first prove various theoretical properties surrounding $\pirad$, and we develop computational mechanisms in Section~\ref{sec:comp_guarantees}.

To this end, our sole assumption will be on the true posterior $\pi\propto\exp(-V)$. Namely, we will assume that $\pi$ is \emph{well-conditioned}: it is log-smooth and strongly log-concave, i.e., for $ \ell_V, L_V > 0,$
\begin{align}\label{well_cond} \tag{$\mathsf{WC}$}
0 \preceq \ell_V I \preceq \nabla^2 V \preceq L_V I\,,
\end{align} 
and minimized at the origin.
These conditions are standard in the theoretical literature on log-concave sampling \citep{ChewiBook} as well as in variational inference \citep[see, e.g.,][]{Lambert2022, arnese2024convergence, Lacker2024, lavenant2024convergence, JiaChePoo24MFVI}.

\begin{remark}
We note that the smoothness of $V$ is not a particularly strong assumption as, e.g., Gaussian mixtures fall into this class. However, the strong convexity of $V$ implies unimodality of the posterior, {which is admittedly stringent. Despite the VI problem being well-posed even without this assumption (Proposition~\ref{existence_uniqueness}), we require it later in order to obtain guarantees for our Wasserstein gradient flow algorithm.}
\end{remark}

\subsection{Regularity of radial minimizer}\label{sec:reg_minimizer}
We first collect some basic properties of the optimal radial minimizer $\pirad^\star$. Proofs of results from this subsection can be found in Appendix~\ref{app:reg_minimizer}.

We first state a result concerning the existence and uniqueness of the minimizer $\pirad^\star$.
\begin{prop}\label{existence_uniqueness}
Assume that there exists $\mu \in \Crad$ with $\kl{\mu}{\pi} < \infty$.
Then, there exists a unique minimizer to \eqref{eq:radvi}.
\end{prop}

Next, we explicitly characterize the stationary condition of the minimizer $\pirad^\star$, where the proof is based on calculus of variations. 
\begin{prop}[Stationary condition]\label{stationary_cond} Suppose $\pi \propto \exp(-V)$ and a minimizer $\pirad^\star$ to \eqref{eq:radvi} exists. Then it holds that $\pirad^\star\propto\exp(-\overline{V})$ with
\begin{align}\label{eq:stationary}
    \overline{V}(x) \defeq \int_{S^{d-1}} V(\|x\|\theta) \dd {\mathsf{Unif}}(\theta)\,,
\end{align}
where ${\mathsf{Unif}}$ is the uniform measure on $S^{d-1}$.
\end{prop}
Using Proposition~\ref{stationary_cond}, we can leverage the existing regularity of $\pi$ (through $V$) to show that $\pirad^\star$ is also regular. In particular, log-smoothness and strong log-concavity are preserved.
\begin{prop}[Regularity of radial minimizer]\label{reg_radial}
Suppose $\pi$ satisfies \eqref{well_cond}. The radial minimizer $\pirad^\star$ also satisfies \eqref{well_cond} with the same parameters.
\end{prop}

\subsection{Regularity of optimal radial maps}\label{sec:reg_ot_map}

We now establish regularity properties of optimal radial (transport) maps from, say, the standard Gaussian to the minimizer of \eqref{eq:radvi}. Our main theorem below is essentially a specialization of Caffarelli's contraction theorem \citep{Caffarelli2000} to the radial component of the optimal transport map.

\begin{theorem}[Regularity of the optimal radial map]\label{theorem_map_regularity} 
Suppose $\pi$ satisfies \eqref{well_cond} and consider the corresponding minimizer to \eqref{eq:radvi}, denoted $\pirad^\star$. Let $T^\star_{\rm{rad}}$ denote the unique optimal transport map from $\rho$ to $\pirad^\star$, and write $T_{\rm rad}^\star(x) = \Psi^\star(\|x\|)\,x/\|x\|$. Writing $r = \|x\|$, the following regularity conditions hold:
\begin{align*}
    \frac{1}{\sqrt{L_V}} \leq (\Psi^\star)'(r) &\leq \frac{1}{\sqrt{\ell_V}}\,, \\
    \left|(\Psi^\star)''(r) \right|
    &\lesssim
    {\frac{\kappa}{\sqrt{\ell_V}}\,\bigl(1 + \frac{d}{r^2}\bigr)\,(1+|r-\sqrt d|)}\,,
\end{align*}
where $\kappa \defeq L_V/\ell_V$.
\end{theorem}

The first result in Theorem~\ref{theorem_map_regularity} follows from a direct application of Caffarelli's contraction theorem; this is made possible by Proposition~\ref{reg_radial}. To prove the second result, we differentiate the Monge--Amp\`ere equation, and exploit the existing regularity of (the first derivative of) $\Psi^\star$. {A key takeaway is that under $\rho$, the norm is tightly concentrated around $r = \sqrt d \pm O(1)$, and thus the bound on $(\Psi^\star)''$ is nearly dimension-free.} While it is possible to differentiate the Monge--Amp\`ere equation a second time and obtain \emph{third}-order control on $\Psi^\star$, we omit this result as it is not necessary for our purposes. The full proof of Theorem~\ref{theorem_map_regularity} can be found in Appendix~\ref{app:reg_ot_map}.

\section{Parametrize then optimize}\label{sec:comp_guarantees}
Recall that our goal is to learn the probability distribution $\pirad^\star$ given query access to the (gradient of the) potential $V$ of the unnormalized posterior $\pi$. Our proposed algorithm will directly learn the optimal transport map from $\rho$ (an easy-to-sample reference measure) to the optimal radial distribution $\pirad^\star$.

We first note the following equivalent optimization problems:
\begin{align*}
    \min_{\mu \in \Crad} \kl{\mu}{\pi} = \min_{T \in \cT_{\rm rad}}\kl{T_\sharp \rho}{\pi}\,,
\end{align*}
where $\rho = \cN(0,I)$. Indeed, any radial distribution $\mu \in \Crad$ can be expressed as the pushforward of the standard Gaussian with any other radial transport map in $\cT_{\rm rad}$. Thus, in this section, we will be focusing on the optimization problem
\begin{align}\label{eq:radvi_t}\tag{$\mathsf{radVI}$-$\mathsf{T}$}
    \qquad \qquad T_{\rm rad}^\star = \argmin_{T \in \cT_{\rm rad}}\kl{T_\sharp\rho}{\pi}\,,
\end{align}
from which we recover the optimal radial distribution.

We follow the approach put forth by \cite{JiaChePoo24MFVI}, who suggest to appropriately \emph{parametrize} the space of transport maps and then to \emph{optimize} over it. Letting {$\{\Psi_j\}_{j=0}^J$} be a fixed set of basis functions and $\alpha > 0$, we consider a parametrized subset of transport maps $\cT_J \subseteq \cT_{\rm rad}$, given by
\begin{align}
\!\!\cT_{J}\!\defeq\!{\Bigl\{ \Bigl(\alpha\,\|
x\| + \sum_{j=0}^{J}\lambda_j \Psi_j(\| x \|)\Bigr) \frac{x}{\norm x} \ \! \Big|\!  \ \lambda \in \R^{J+1}_+ \Bigr\}}\,.
\end{align}
In the rest of this section, we will address the following questions, which naturally arise as a result of our choice of parametrization:

\begin{itemize}
    \item  What is the proposed basis ${\{\Psi_j\}_{j=0}^J}$? How large does $J$ need to be such that $\cT_J$ is a faithful approximation to $\cT_{\rm rad}$?
    \item Does optimizing over $\cT_J$ yield a map which is close to $T^\star_{\rm rad}$? Is it possible to obtain optimization guarantees? 
\end{itemize}

\subsection{Approximation guarantees}\label{sec:approx_guarantees}
We first describe our choice of basis. As an arbitrary function ${\Psi}$ is a continuous, strictly monotone function, we parametrize this class of functions by a piecewise linear monotone curve. Given a cutoff variable $R > 0$, we consider the following sequence of equi-spaced piecewise monotone functions on the interval $[\sqrt d - R,\, \sqrt d + R] \subset \R$
\begin{align}\label{eq:basis_functions}
    {\Psi}_j(r) \defeq {\Psi}_{\rm base}(\delta_j^{-1}(r-a_j))\,,
\end{align}
where for $j\ge 1$, $\delta_j \defeq \delta$ is the mesh size, $\{a_j\}_{j=1}^J$ are the knots; and for $j=0$, $\delta_0 \defeq \sqrt d - R$ and $a_0 \defeq 0$.
Here, $\Psi_{\rm base}(x) \defeq 0 \vee (x \wedge 1)$.
Ultimately, we will choose {$J = 2R/\delta + 1$ (we assume that $R$ is divisible by $\delta$ and that $R \le \sqrt d$)}.

Our first result of this section establishes a ``universality'' property of the set $\cT_J$ from an approximation perspective, and yields the choice of $R$ and $J$ required to complete our construction. {Since we mainly care about high-dimensional approximation, henceforth we assume $d\ge 3$.}

\begin{theorem}[Universal approximation]\label{theorem_unif_approx}
Let $T^\star \in \cT_{\rm rad}$ which satisfies Theorem~\ref{theorem_map_regularity}, and let ${\epsilon \gg \exp(-\Omega(d))}$. Define $\cT_J$ with $\alpha = 1/\sqrt{L_V}$ and choose $R \asymp {\sqrt{\log(d/\epsilon)}}$, $J = {\widetilde{\Omega}(\sqrt{\kappa/\eps})}$ with the basis elements given by \eqref{eq:basis_functions}. Then there exists a $\hat{T} \in \cT_J$, i.e., there exists $\hat{\lambda} \in \R^{J+1}_+$ with $\hat{T} = T_{\hat{\lambda}}$, such that
\begin{align*}
    \|T^{\star} - \hat{T}\|_{L^2(\rho)} &\leq \epsilon/\ell_V^{1/2}\,,\\
    \|D(T^{\star} - \hat{T} ) \|_{L^2(\rho)} &\le {\widetilde{\mathcal{O}}(\kappa^{1/2} \epsilon^{1/2}/\ell_{V}^{1/2})}\,.
\end{align*}
\end{theorem}

\subsection{Proposed algorithm: \texttt{radVI}}\label{sec:algo}
We now present our basic algorithm. Recall that our objective function, the KL divergence, is
\begin{align*}
\begin{split}
    \kl{\mu}{\pi} &= \cV(\mu) + \cH(\mu) + \log Z\,, \\
    &\defeq \int V \dd \mu + \int \log \mu \dd \mu + \log Z\,,
\end{split}
\end{align*}
where $Z > 0$ is the unknown normalizing constant of $\pi$. If we write $\mu = T_\sharp\rho$, for some transport map $T$, then by a change-of-variables calculation, one obtains (up to omitted constants)
\begin{align*}
\begin{split}
        \kl{T_\sharp\rho}{\pi} &= \int V\circ T\dd \rho
        - \int \log \det  DT \dd \rho\,.
\end{split}
\end{align*}
We now optimize over our prescribed parametrization. For $T_\lambda \in \cT_J$, we see that we can write the KL divergence as a function over the non-negative orthant:
\begin{align}\label{eq:cF_obj}
\lambda \mapsto \cF(T_\lambda) \defeq \kl{(T_\lambda)\sharp\rho}{\pi} \quad (\lambda \in \R^{J+1}_+)\,.
\end{align}
To continue, we require the following observation: there exists an isometry between the $L^2(\rho)$ distance on the transport maps $T_\lambda \in \cT_J$ and a Euclidean distance over the weights $\lambda \in \R^{J+1}_+$. Indeed, for any two parameters $\lambda,\eta \in \R^{J+1}_+$, one readily computes
\begin{align*}
\|T_\lambda - T_\eta\|^2_{L^2(\rho)}
&= \|\textstyle\sum_{j=1}^J(\lambda_j - \eta_j)\,\Psi_j(\|x\|)\|^2_{L^2(\rho)} \\
&= (\lambda - \eta)^\top Q(\lambda - \eta)\,,
\end{align*}
where $Q \in \mathbb{S}^{J+1}_{+}$ is a Gram matrix with entries given by\looseness-1
\begin{align*}
    Q_{i,j} \defeq \E_{X \sim \rho}[\Psi_i(\|X\|)\,\Psi_j(\|X\|)]\,.
\end{align*}
{We detail in Appendix~\ref{app:gram_construction} how to compute the entries $Q_{i,j}$ via truncated moments of the chi distribution.} Thus, convergence of the radial maps corresponds to convergence of the parameters in a \emph{weighted} Euclidean space, namely $(\R^{J+1}_+,\|\cdot\|_Q)$. In other words, a discretization of a gradient flow of \eqref{eq:cF_obj} would naturally correspond to \emph{projected gradient descent} in the weighted metric $\|\cdot\|_Q$:
\begin{align}\label{eq:det_iterates}
    \lambda^{(k+1)} = {\rm Proj}_{\R^{J+1}_+,\|\cdot\|_Q}\!\bigl(\lambda^{(k)} - h Q^{-1}\nabla_\lambda \cF(T_{\lambda^{(k)}})\bigr)\,,
\end{align}
where $h > 0$ is the stepsize, and $\nabla_\lambda \cF(T_\lambda)$ is the gradient of the objective function in the parameters $\lambda$. \looseness-1

Our complete algorithm, called \texttt{radVI}, is presented in Algorithm~\ref{alg:radVI}. Note that in practice, a stochastic estimate of the gradient will be used in place of the full gradient. Thus, \texttt{radVI} can be seen as a special instance of stochastic projected gradient descent (SPGD). We discuss this more in Section~\ref{sec:stoc_opt}. 
\begin{remark}
    It is worth stressing that the choice of using gradient descent as a first-order algorithm was entirely arbitrary, and many other algorithms (e.g., Frank--Wolfe, mirror descent, etc.) are applicable within our framework. The functional of interest can be suitably arbitrary as well; see \cite{JiaChePoo24MFVI} for more details.
\end{remark}
\begin{algorithm}[t]
\caption{\texttt{radVI}}\label{alg:radVI}
\begin{algorithmic}
\State{\textbf{Input:}} Posterior $\pi \propto \exp(-V)$ with access to $\nabla V$
\State{\textbf{Free parameters:}} Choose $K, h > 0$
\State\textbf{Construct:} Basis family $\{\Psi_j\}_{j=0}^J$ and matrix $Q$
\State{\textbf{Initialize:}} $\lambda^{(0)} \in \R^{J+1}_+$
\While{$k=0,1,\ldots,K-1$}

\begin{align*}
    &\text{Compute stochastic gradient } \widehat\nabla_\lambda \cF(T_{\lambda^{(k)}}) \\
    &\lambda^{(k+1)} = {\rm Proj}_{\R^{J+1}_+,\|\cdot\|_Q}\bigl(\lambda^{(k)} - hQ^{-1}\widehat\nabla_\lambda \cF(T_{\lambda^{(k)}})\bigr)
\end{align*}

\EndWhile{}
\State \textbf{Return:} $T_{\lambda^{(K)}}$
\end{algorithmic}
\end{algorithm}
\vspace{-4mm}
\subsection{Optimization guarantees}\label{sec:opt_guarantees}
The main result of this section is the following quantitative convergence of \texttt{radVI} to the optimal radial map. 
{\begin{theorem}[Convergence of \texttt{radVI}]\label{thm:radvi_converges} Suppose $\pi$ satisfies \eqref{well_cond}, {and consider the family of transport maps constructed via \eqref{eq:basis_functions}.} Then, the transport map $T_{\lambda^{(K)}}$ with $(\lambda^{(k)})_{k\ge 0}$ given by \eqref{eq:det_iterates} is $\epsilon$-close to $T_{\rm rad}^\star$ (in $L^2(\rho)$) so long as $J = \tilde\Theta(\kappa^2/\eps)$, the step size $h = \widetilde\Theta(\eps/(L_V\kappa^2))$ and for iterations
\begin{align*}
     K = \tilde\Omega\bigl(\kappa^5\eps^{-1}\log(\kl{(T_{\lambda^{(0)}})_\sharp\rho}{\pi}/\eps^2)\bigr)\,.
\end{align*}
\end{theorem}}
{We outline a proof sketch, leaving the fine details to Appendix~\ref{app:radvi_converges_proof}.} Let us first assume that the sequence $(\lambda^{(k)})_{k \geq 0}$ eventually converges to some optimal $\lambda^\star \in \R^{J+1}_+$. If $\pi$ satisfies \eqref{well_cond}, then the corresponding $T^\star_{\rm rad}$ satisfies the Theorem~\ref{theorem_map_regularity}. Also, by Theorem~\ref{theorem_unif_approx}, there exists $\hat{\lambda} \in \R^{J+1}_+$ such that $T_{\hat{\lambda}} \in \cT_J$ is $\epsilon$-close to $T^\star_{\rm rad}$ for $J$ sufficiently large. By triangle inequality, we have
\begin{align*}
    &\| T_{\lambda^{\star}} - T^\star_{\rm rad}\|^2_{L^2(\rho)} \\
    &\qquad \lesssim \| T_{\lambda^{\star}} - T_{\hat\lambda}\|^2_{L^2(\rho)} + \| T_{\hat\lambda} - T^\star_{\rm rad}\|^2_{L^2(\rho)} \\
    &\qquad \lesssim \| T_{\lambda^{\star}} - T_{\hat\lambda}\|^2_{L^2(\rho)} + {\epsilon^2}\,.
\end{align*}
Appealing to the strong convexity of the KL divergence along generalized geodesics, one can show that the remaining term can be bounded by precisely \emph{both} terms from Theorem~\ref{theorem_map_regularity}
\begin{align}\label{eq:helper_inequality}
\begin{split}
    \| T_{\lambda^{\star}} - T_{\hat\lambda}\|^2_{L^2(\rho)} &\lesssim {\kappa}\, \| T_{\hat\lambda} - T^\star_{\rm rad}\|^2_{L^2(\rho)} \\
    &\quad + {\kappa^2}\, \| D(T_{\hat\lambda} - T^\star_{\rm rad})\|^2_{L^2(\rho)}\,.    
\end{split}
\raisetag{25pt} 
\end{align}
See \citet[Appendix C]{JiaChePoo24MFVI} for a proof of this fact (in particular, the proof of their Theorem 5.9 and Corollary C.3). Thus, for $J$ large enough, the minimizer over $\cT_J$, i.e., $T_{\lambda^\star}$, can be made arbitrarily close to $T^\star_{\rm rad}$. To conclude, it remains to quantitatively assess that $T_{\lambda^{(k)}} \to T_{\lambda^\star}$, and then string everything together.\looseness-1

By the isometry property, it holds that
\begin{align*}
    \| T_{\lambda^{(k)}} - T_{\lambda^\star} \|^2_{L^2(\rho)} = \|\lambda^{(k)} - \lambda^\star\|^2_{Q}\,,
\end{align*}
where we recall that $Q \succ 0$ is the fixed Gram matrix defined from the basis functions. If $\lambda \mapsto \cF(T_\lambda)$ is smooth and strongly convex (with respect to $\|\cdot\|_Q$), we can readily apply existing results for the convergence of first-order algorithms \citep[see, e.g.,][]{Beck2017}. Thankfully, the requisite properties of our functional can be verified; {see Appendix~\ref{app:proof_smooth_cvx}.}
\begin{prop}\label{prop:cf_smooth_convex}
Suppose $\pi$ satisfies \eqref{well_cond}, and also consider $\cT_J$ chosen as in Theorem~\ref{theorem_unif_approx}. Then $\lambda\mapsto \cF(T_\lambda)$ is $\ell_V$-strongly convex and ${\widetilde O(J^2 L_V)}$-smooth (with respect to $\|\cdot\|_Q$).
\end{prop}
\begin{remark}
Note that the smoothness constant of $\lambda \mapsto \cF(T_\lambda)$ explodes as $J \to \infty$. This is unsurprising, as the functional $\mu\mapsto \kl{\mu}{\pi}$ is not smooth over the space of all probability measures. {In contrast, our parametrization creates a strict subset of all probability measures, over which the constant can be controlled.} 
\end{remark}

\subsection{Stochastic optimization}\label{sec:stoc_opt}
In practice, we use \emph{stochastic} projected gradient descent (SPGD) to optimize $\lambda \mapsto \cF(T_\lambda)$ over $(\R^{J+1}_+,\|\cdot\|_Q)$. Recall our full objective function (up to constants) is
\begin{align*}
    \cF(T_\lambda)\! = \!\int V(T_\lambda(x))\dd\rho(x) \!- \!\int \log\det(DT_\lambda(x))\dd\rho(x)\,.
\end{align*}
We compute the gradient of the two terms separately. For the first term, we pass the gradient in the weights $\lambda$ under the expectation and, due to the definition of $T_\lambda$, the inner gradient is simply
\begin{align}\label{eq:nablaV_lambda}
    \nabla_\lambda V(T_\lambda(X)) = {\vec\Psi(\|X\|)(X/\|X\|)^\top}\nabla V(T_\lambda(X))\,,
\end{align}
where $\vec\Psi(r) \defeq [\Psi_0(r),\ldots,\Psi_J(r)]$.
Thus, to compute $\nabla_\lambda \E_{X\sim\rho}[V(T_\lambda(X))]$, it suffices to use a Monte Carlo approximation using i.i.d.~draws $X_1,\ldots,X_N \sim \rho$.

For the second term, we first compute
\begin{align*}
    \log\det(DT_\lambda(x)) &= (d-1)\log(\alpha+\langle\lambda,\Vec{\Psi}(\|x\|)\rangle /\|x\|) \\
    & + \log(\alpha + \langle\lambda,\vec\Psi'(\|x\|)\rangle)\,,
\end{align*}
where $\Vec{\Psi}'(r) \defeq [\Psi_0'(r),\ldots,\Psi'_J(r)]$. It is straightforward to show that $\nabla_\lambda \log\det(DT_\lambda(x))$ has an analytical expression, and thus the integrated expression can again be computed via Monte Carlo integration. However, due to the precise nature of $\vec\Psi$ and $\vec\Psi'$, it is possible to evaluate our second integrated quantity in our objective using {simple univariate numerical integration}; we briefly touch on this in {Appendix~\ref{app:grad_objective}}.

The next theorem states that \texttt{radVI} still comes with convergence guarantees when using stochastic gradient estimates (e.g., Monte Carlo estimates) for \eqref{eq:nablaV_lambda}. Unlike many works that use SPGD, we \emph{prove} that under the well-conditioned assumption, we satisfy a classical \emph{bounded variance} property which permits us to use existing theory; see Appendix~\ref{app:sgd_radvi_converges_proof}
for a proof. 
\begin{theorem}[Convergence of stochastic \texttt{radVI}]\label{thm:sgd_radvi_converges}
    Assume that $\pi$ is well-conditioned ~\ref{well_cond} and consider the family of transport maps constructed via \eqref{eq:basis_functions}. Then, {for all sufficiently small $\eps$,} the iterates of stochastic projected gradient descent yield a measure $\mu_{(t)} $ with the guarantee ${\ell_V}\E\|T_{\lambda^{(k)}} - T_{\rm rad}^\star\|^2_{L^2(\rho)} \leq \eps^2$, with a number of iterations bounded by
    \begin{align*}
        K = \tilde\Omega \Bigl({d\kappa^2 J^3L_V}\eps^{-2}\log(\|T_{\lambda^{(0)}} - T_{\rm rad}^\star\|^2_{L^2(\rho)} /\eps)\Bigr)\,,
    \end{align*}
    and step size $h = \tilde\Theta(\eps^2/(d\kappa^3 J^3L_V))$. Moreover, if we choose the parameters of our dictionary as in Theorem~\ref{theorem_unif_approx}, then we achieve the same $\eps$-closeness above with
    \begin{align*}
        K = \tilde\Omega\Bigl({d\kappa^{5/2} }\eps^{-5}\log(\|T_{\lambda^{(0)}} - T_{\rm rad}^\star\|^2_{L^2(\rho)}/\eps)\Bigr)\,.
    \end{align*}
\end{theorem}

\subsection{\texttt{radVI} with whitening}\label{ssec:whiten}
\begin{algorithm}[t]
\caption{ \texttt{radVI} with whitening}\label{alg:radVI_precond}
\begin{algorithmic}
\State{\textbf{Input:}} Posterior $\pi \propto \exp(-V)$ with access to $\nabla V$
\State\textbf{Whitening stage:}

\begin{enumerate}
    \item Fix $(m,\Sigma) \in \R^d \times \mathbb{S}^d_{+}$ 
    \item Define
    $x \mapsto T_{m,\Sigma}(x) \defeq \Sigma^{1/2}x + m$
    \item Modify posterior via 
    $\tilde{V} \gets V \circ T_{m,\Sigma}$
\end{enumerate}

\State\textbf{Obtain:} $\hat T_{\rm rad} \gets \texttt{radVI}(\exp(-\tilde V))$
\State\textbf{Return:} Composite map 
       $T_{\mathrm{comp}} \gets T_{m,\Sigma} \circ \hat{T}_{\mathrm{rad}}$
\end{algorithmic}
\end{algorithm}
As we highlighted in the introduction, a strength of \texttt{radVI} is that it can be used in conjunction with other variational methods based on, say, Gaussian distributions, such as Gaussian VI, mean-field Gaussian VI, and Laplace approximation. Algorithm~\ref{alg:radVI_precond} illustrates how to use any of these off-the-shelf Gaussian approximation methods to whiten the target distribution and improve performance. 

In summary, any Gaussian measure $\cN(m,\Sigma)$ can be expressed as $(T_{m,\Sigma})_\sharp\rho$ where $\rho = \cN(0,I)$ and $T_{m,\Sigma}(x) = \Sigma^{1/2}x + m$. Thus, given any Gaussian approximation to $\pi$, we can use the corresponding affine map to whiten the posterior $\pi\propto\exp(-V)$ by defining $\tilde{V} \defeq V\circ T_{m,\Sigma}$.
Then, we run \texttt{radVI} on $\exp(-\tilde{V})$, and output the composition of these two maps. Below, we briefly review two standard Gaussian approximation methods in the literature.
\vspace{-6mm}
\paragraph{Laplace approximations (LA).}
The Laplace approximation method occurs in two stages. For a posterior $\pi \propto \exp(-V)$, we first compute the mode of the distribution, $x^\star \in \argmin_{x \in \R^d} V(x)$,
and then we compute $(\nabla^2 V(x^\star))^{-1}$. The final approximation to the posterior $\pi$ is then
\begin{align}\label{eq:pi_la}
    \pi_{\rm LA} = \cN(x^\star, (\nabla^2 V(x^\star))^{-1})\,.
\end{align}
Note that this method fails if $\nabla^2 V(x^\star)$ is not invertible, and is known to be inaccurate when $\pi$ is heavily skewed (when the mode is far away from the mean). See \cite{katsevich2023laplace,katsevich2024laplace} for recent statistical developments.
\paragraph{Gaussian VI (GVI).}
In Gaussian VI, the practitioner replaces $\Crad$ in \eqref{eq:radvi} with $\cN$, the set of all normal distributions with positive definite covariance. The resulting optimization problem becomes
\begin{align}\label{eq:gaussian_vi}
    \pi_{\rm GVI} \in \argmin_{\mu \in \cN}\kl{\mu}{\pi}\,;
\end{align}
see \citet{Lambert2022,Diao2023,kim2023convergence}. A major limitation to GVI is the storage and per-iteration complexity, as the running covariance estimate needs to be stored and inverted at each iteration, which is costly for $d \gg 1$. Optimizing over \emph{product} Gaussian measures (Gaussian mean-field VI, or MFVI) can mitigate these numerical hurdles, reducing the per-iteration complexity from $\cO(d^3)$ to $\cO(d)$ at the cost of possibly being much farther from $\pi$.
\begin{remark}
    We briefly mention the work of \cite{liang2022fat}, which similarly tries to adjust the tail distributions of their approximated distribution. However, their approach is (i) parametric in nature, and (ii) builds off the mean-field variational inference perspective. Specifically, they only consider product measure approximations to the posteriors.
\end{remark}
\section{Experiments}\label{sec:experiments}
In all experiments, $d=50$, we choose $\alpha=0.01$ and $\lambda^{(0)} = \bm{1}_J$, $R = \sqrt{\log d}$, and $\delta = d^{-1/6}$  as parameters for constructing our dictionary. For LA, we use a standard optimization solver and closed-form expressions of the gradient and Hessian. For GVI, we use the Forward–Backward method of \citet{Diao2023}.\footnote{The code used for our implementation of \texttt{radVI} and the numerical experiments are made available at \href{https://github.com/gluca99/radVI}{github.com/gluca99/radVI}.}

\begin{figure}[t]
    \centering
\includegraphics[width=0.4\textwidth]{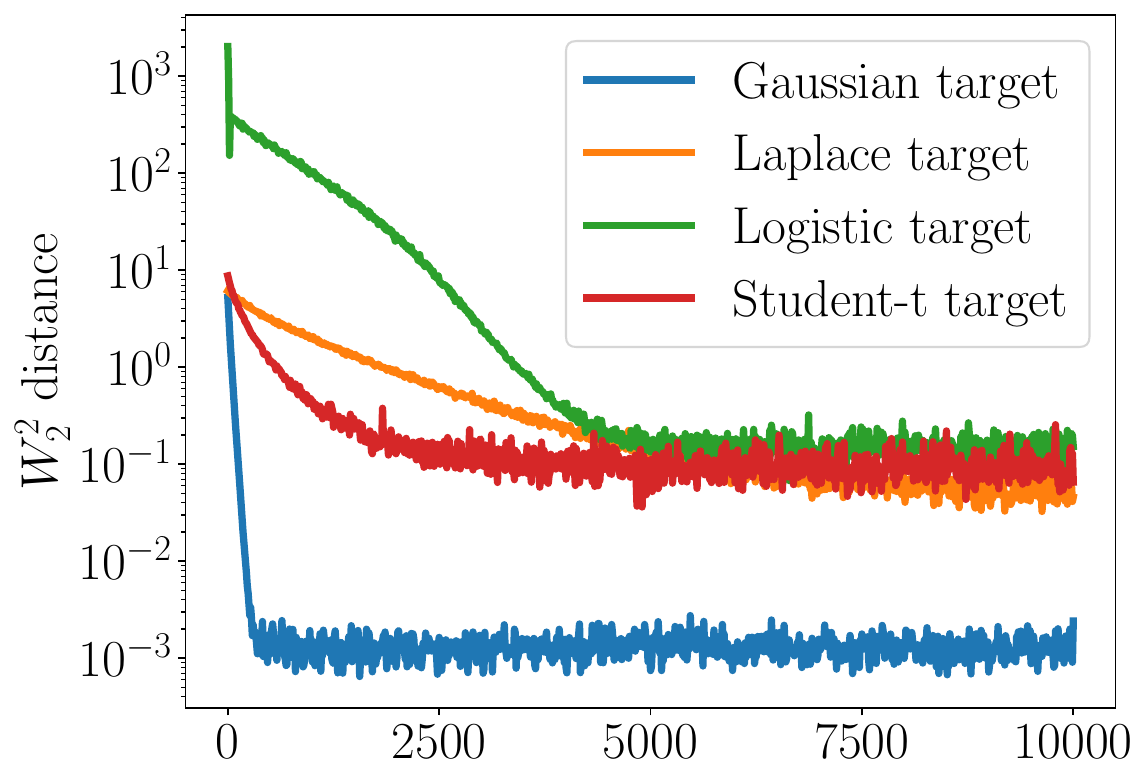}
    \caption{Convergence of \texttt{radVI} for various target distributions. See Table~\ref{tab:table1} for final-iterate comparisons between GVI and LA.}
    \label{fig:wass_ma_figure}
\end{figure}

The candidate posteriors for the majority of this section are the Student-$t$, Laplace, and logistic distributions. We remark that none of these distributions fully satisfy our requirements. For instance, none of them are strongly convex (in fact, the Student-$t$ distribution is non-log-concave), and the Laplace distribution is non-smooth. For Student-$t$, we use 10 degrees of freedom, which leads to significantly heavier tails than a Gaussian. Nevertheless, we are able to show that our scheme can recover the desired target distribution to better accuracy than standard VI methods. The precise definitions of these distributions, their potential functions, etc., can be found in {Appendix~\ref{app:analytical_ot_map}}. In Section~\ref{sec:neal}, we study parameter estimation from Neal's funnel distribution, a standard hierarchical prior.

\begin{table}[t]
\centering
\resizebox{\linewidth}{!}{%
\begin{tabular}{lcccc}
& \multicolumn{4}{c}{Isotropic targets} \\
\cmidrule(lr){2-5}
Method & Gaussian &  Laplace & Logistic  & Student-$t$ \\
\toprule
LA     &   {$2.45 \times 10^{-4}$}    &  $20.00$      &  $1.6 \times  10^{3}$     & $25.87$  \\
GVI    &   $7.34 \times 10^{-4}$     &   $8.24$     &   $3.96$    & $1.99$ \\
\texttt{radVI}  &  \bm{$1.15  \times 10^{-4}$}     &  \bm{$5.37 \times 10^{-2}$}      &   \bm{$1.84 \times 10^{-1}$} & \bm{$1.19 \times 10^{-1}$} \\
\bottomrule
\end{tabular}}
\caption{Estimated {squared} Wasserstein distance between various VI solutions for learning isotropic targets.}\label{tab:table1}
\end{table}

\begin{figure}[t]
\centering
\hspace{-0.1cm}
\includegraphics[width=0.42\textwidth]{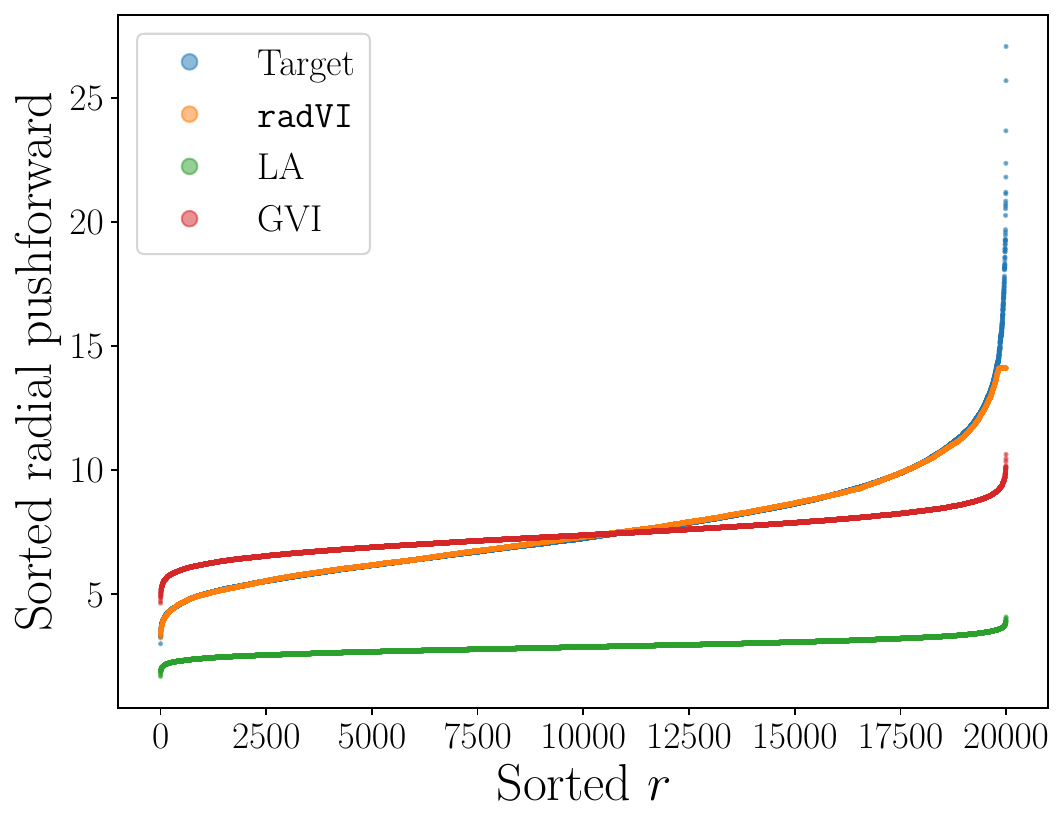}
\includegraphics[width=0.42\textwidth]{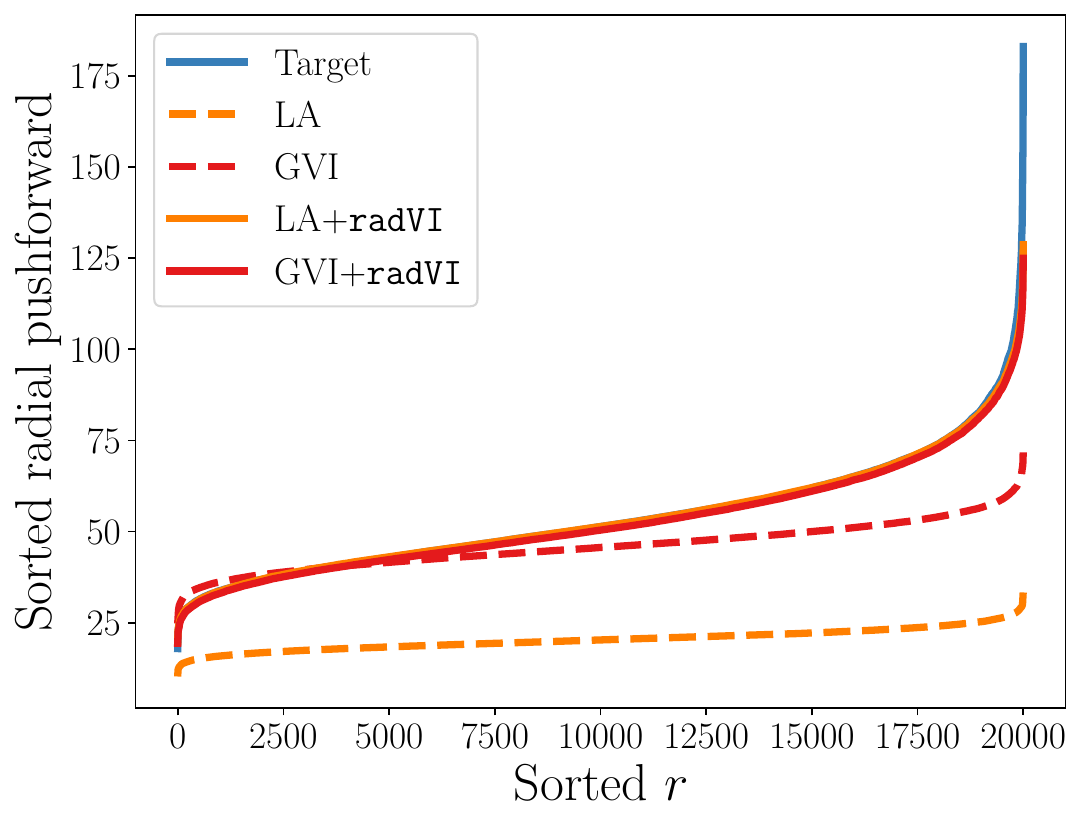}
\caption{Comparing learned radial profiles of \texttt{radVI} versus other approximation methods for learning the Student-$t$ distribution in the isotropic (\textbf{top}) and anisotropic case \textbf{(bottom)}.}
\label{fig:isotropic_laplace_figure}
\end{figure}

\subsection{Learning isotropic radial families}\label{sec:isotropic}
We first investigate the case where the target measure is isotropic and radially symmetric. As $\pi = \pirad^\star = (T_{\rm rad}^\star)_\sharp\rho$, for any iterate $\lambda^{(k)}$ in our algorithm, we can approximate the $L^2$ distance between the maps via
\begin{align*}
\|T_{\lambda^{(k)}} - T^\star_{\rm rad}\|^2_{L^2(\rho)} \!\simeq \!\frac{1}{n}\sum_{i=1}^{n}\|T_{\lambda^{(k)}}(X_i) - T_{\rm rad}^\star(X_i)\|^2\,,
\end{align*}
where $X_i \sim \rho$, and we use $n = 10^4$ to estimate all quantities. {$T^\star_{\rm rad}$ is known in closed form for the Gaussian case and Student-$t$ distribution, but must be solved numerically for the logistic and Laplace distributions; see Appendix~\ref{app:analytical_ot_map} for a short explanation. We compare the (squared) $2$-Wasserstein distance between ground truth samples and those generated by LA and GVI.} {See Figure~\ref{fig:wass_ma_figure} for a plot comparing the convergence of these various methods.} As a sanity check, our method recovers the Gaussian distribution with an error tolerance comparable to Gaussian methods. For the heavy-tailed Student-$t$ distribution, however, we outperform these methods by wide margins. {Repeating the experiment in $d=100$ gives similar results (see Appendix~\ref{sec:high_dim_examples}).}

It is perhaps more informative to visually distinguish the approximation schemes. To this end, we plot the radial profiles of the various approximation methods. For instance, the top of Figure~\ref{fig:isotropic_laplace_figure} compares the various learned profiles for the Student-$t$ distribution. Unlike LA or GVI, \texttt{radVI} can find a close radial profile to the target. {The corresponding figures for isotropic Laplace and logistic distributions appear in Appendix~\ref{sec:aux_figures}.}\looseness-1 

{Finally, we mention that we additionally performed a simple sensitivity analysis regarding our hyperparameter $\alpha$ similar to Figure 2 of \cite{JiaChePoo24MFVI}. Focusing on the isotropic Gaussian case in which the ground-truth value is $\alpha = 1/\sqrt{L_V}=1$, we run our algorithm where we vary $\alpha \in \{10^0, 10^{-1}, 10^{-2}\}$. Figure~\ref{fig:conv_varying_alpha} shows  that \texttt{radVI} converges to the optimal parameters exponentially fast, highlighting how our algorithm is robust to the choice of $\alpha$.}
\begin{figure}[!h]
    \centering
    \includegraphics[width=\linewidth]{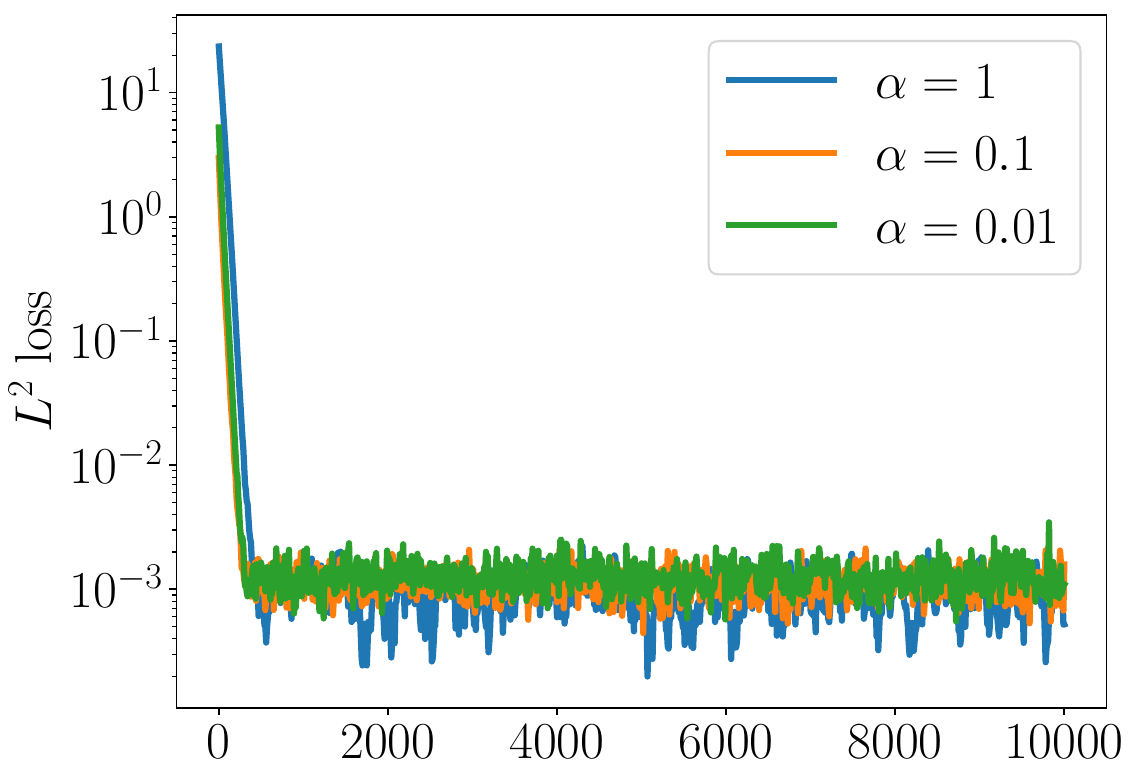}
    \caption{In the case where $\pi$ is an isotropic Gaussian with $d=50$, we verify that \texttt{radVI} is robust to the choice of $\alpha$.}
    \label{fig:conv_varying_alpha}
\end{figure}

\subsection{Learning anisotropic distributions}\label{sec:anisotropic}
We now consider the anisotropic setting, where the distribution has a randomly generated covariance parameter $\Sigma = AA^\top + I$, where $A_{ij}\sim\cN(0,1)$ and $m=0$. 
Following Algorithm~\ref{alg:radVI_precond}, we first run either LA or GVI to create Gaussian approximations, and then we obtain our complete composite map, allowing us to draw as many samples as desired. Figure~\ref{fig:anisotropic_radial_profiles} showcases performance on an anisotropic logistic where we visualize generated samples. The Gaussian approximation methods fail to capture the correct tail behavior, while the whitened \texttt{radVI} approximations do. We observe the same performance for the anisotropic Student-$t$ and Laplace distributions; see the bottom of Figure~\ref{fig:isotropic_laplace_figure} and Appendix~\ref{sec:aux_figures}.\looseness-1
\begin{figure}[t]
\centering
\hspace{-0.1cm}
\includegraphics[width=0.4\textwidth]{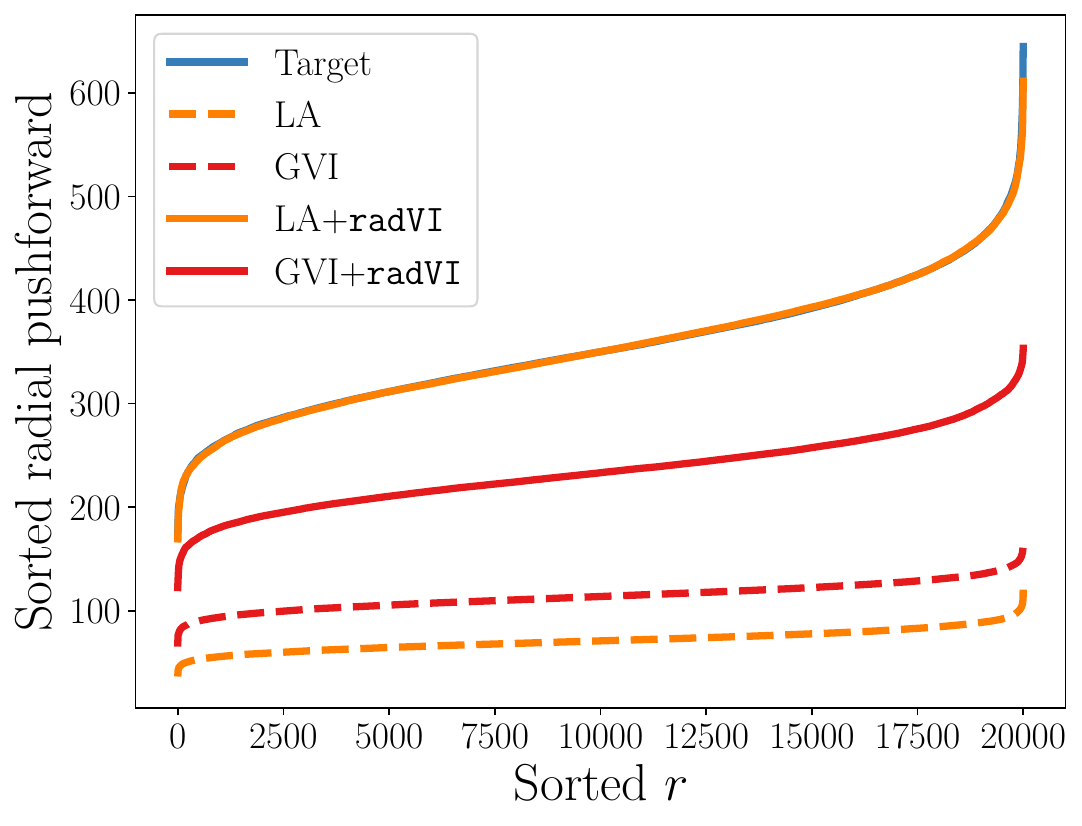}
\includegraphics[width=0.4\textwidth]{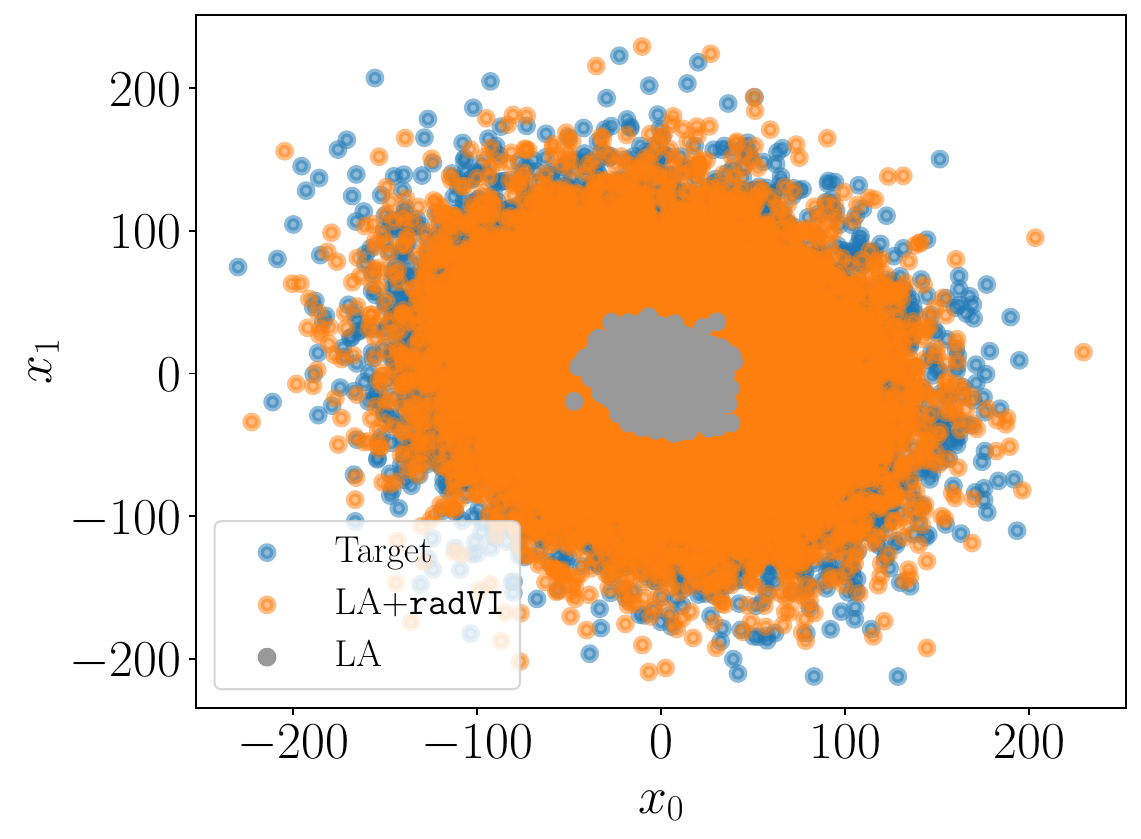}
\caption{{\textbf{Top:} Comparing whitening methods for learning the anisotropic logistic distribution, with and without \texttt{radVI}}. \textbf{Bottom:}
Visual comparison of true target samples, those generated by LA, and ours (LA+\texttt{radVI}). 
}
\label{fig:anisotropic_radial_profiles}
\end{figure}

\subsection{\texttt{radVI} for parameter estimation}\label{sec:neal}
{We now demonstrate how \texttt{radVI} can correct tail-underestimation for certain posteriors $\pi$. To illustrate this, first let $f : \R^d \to \R$ be a functional of interest and consider the problem of estimating
\begin{align*}
    \E_{X\sim\pi}[f(X)]\,.
\end{align*}
Letting $\widehat\pi$ denote a generic variational approximation, note that we always have the approximation
\begin{align*}
    \E_{X \sim \pi}[f(X)] &= \E_{Y \sim \widehat\pi}[f(Y)\pi(Y)/\widehat\pi(Y)] \\
    &\approx \frac{1}{n}\sum_{i=1}^n [f(Y_i)\pi(Y_i)/\widehat\pi(Y_i)]\,,
\end{align*}
where $Y_i \sim \widehat\pi$ are easily-drawn samples.}

{To investigate how \texttt{radVI} can be used to improve GVI for parameter estimation, we study Neal's funnel distribution, a common example in the VI literature \citep{neal2003slice,betancourt2015hamiltonian}. For $d > 1$, suppose $z \sim \cN(0,4)$ and $x_i \sim \cN(0,e^z)$ for $i \in [d]$; one can explicitly write $\pi \in \cP(\R^{d+1})$ and compute the corresponding log-density and its derivatives. We stress that this model is misspecified for radial distributions, and instead fall under the category of a \emph{structured} posterior \citep{sheng2025theory}.} We follow related work and estimate the following quantities from the posterior: $\E[z^2] = 4$, $\E[x_1^2] \approx 7.389$, and $\PP(|z| > 2) \approx 0.317$. 

In Table~\ref{tab:ess}, we report our results. We drew 2000 samples from both GVI and GVI+\texttt{radVI}, and reported the estimated parameters averaged over 1000 trials and also computed the standard error (averaged across the trials). As expected, the parameter estimates are significantly better when incorporating \texttt{radVI}, which we stress comes with minimal computational cost. For instance, standard Gaussian VI will report that the tail probability is identically zero whereas our estimate using \texttt{radVI} is significantly closer to the ground truth.

\begin{table}[t]
\centering
\resizebox{\linewidth}{!}{%
\begin{tabular}{lccc}
& \multicolumn{3}{c}{\textbf{$d=25$}} \\
\cmidrule(lr){2-4}
Parameter & $\mathbb{E}[z^2] = 4$  & $\mathbb{E}[x_1^2] \approx 7.389$  & $\mathbb{P}(|z|>2) \approx 0.317$ \\
\toprule
GVI & 0.274 $\pm$ $4\times10^{-3}$ & 1.12 $\pm$ $1.4\times 10^{-2}$ & 0 \\
GVI+radVI & $\bm{2.41}$ $\pm$ $2\times 10^{-2}$ & $\bm{5.96}$ $\pm 1.0\times 10^{-1}$ & $\bm{0.214}$ $\pm$ $3\times 10^{-3}$ \\
\bottomrule
\\[-0.9em]
& \multicolumn{3}{c}{\textbf{$d=50$}} \\
\cmidrule(lr){2-4}
Parameter & $\mathbb{E}[z^2] = 4$  & $\mathbb{E}[x_1^2] \approx 7.389$  & $\mathbb{P}(|z|>2) \approx 0.317$ \\
\toprule
GVI & 0.328 $\pm$ $5\times10^{-3}$ & 1.61  $\pm$ $2\times10^{-2}$ & 0 \\
GVI+radVI & $\bm{2.51}$  $\pm$ $4\times10^{-2}$ & $\bm{6.48}$  $\pm$ $2\times10^{-1}$ & $\bm{0.19}$  $\pm$ $5\times10^{-3}$ \\
\bottomrule
\end{tabular}}
\caption{In $d\in\{25,50\}$, we compare the performance of GVI and GVI+\texttt{radVI} for parameter estimation.}
\label{tab:ess}
\end{table}
\vspace{-4mm}
\section{Conclusion}
We propose and analyze a framework for variational inference over the space of \emph{radial} transport maps, leading to our algorithm, \texttt{radVI}. Under standard assumptions, we prove convergence guarantees for our algorithm in both the deterministic and stochastic settings. Our analysis hinges on novel regularity properties of optimal transport maps between radially symmetric distributions. We demonstrate our ability to learn radial distributions in a suite of experiments where standard VI methods fail to capture the correct behavior. An interesting open question is to lift the log-concave assumptions and still obtain optimization guarantees in this setting, as many have in the sampling literature \citep{balasubramanian2022towards, ChewiBook}.

\subsubsection*{Acknowledgements}
AAP thanks the Yale Institute for the Foundations of Data Science for financial support.

\bibliography{references}

\appendix
\onecolumn
\section{Omitted proofs from Section~\ref{sec:reg_minimizer}}\label{app:reg_minimizer}

\subsection{Proof of Proposition~\ref{existence_uniqueness}}

{The existence of a minimizer follows from standard arguments, since the KL divergence is weakly lower semicontinuous, has (weakly) compact sublevel sets, and $\Crad$ is (weakly) closed.} {Uniqueness follows from strict convexity of the KL divergence, together with convexity of $\Crad$: if $\mu,\nu \in \Crad$, then $\frac{1}{2}\,(\mu+\nu) \in \Crad$, since a mixture of two radial measures is radial.}

\subsection{Proof of Proposition~\ref{stationary_cond}}
Recall that
\begin{align*}
    \kl{\mu}{\pi} = \int V(y)\dd\mu(y) + \int \log(\mu(y))\dd\mu(y)\,,
\end{align*}
where $\mu \in \Crad$. As $\mu$ is radial, we can express the above in polar coordinates as
\begin{align*}
    {s_d^{-1}}\kl{\mu}{\pi} = \int_0^\infty \int_{S^{d-1}} V(r\theta) \,\mu(r)\, r^{d-1}\dd{\mathsf{Unif}}(\theta) \dd r + \int_0^\infty \log \mu(r)\, \mu(r)\, r^{d-1}\dd r\,,
\end{align*}
{where $\mu(r) \defeq \mu(r\theta)$ for some (thus all) $\theta \in S^{d-1}$, and $s_d$ denotes the volume of $S^{d-1}$.}
Taking the first variation in $\mu$, one computes
\begin{align*}
    \int_{S^{d-1}}V(r\theta) \dd {\mathsf{Unif}}(\theta) + \log \mu(r) = {\text{constant}}\,.
\end{align*}
Rearranging yields the claim.

\subsection{Proof of Proposition~\ref{reg_radial}}

We first require the following computation.
\begin{lemma}
If $V$ is $L_V$-smooth and $\ell_V$-strongly convex, then $r\mapsto \overline{V}(r) \defeq \int_{S^{d-1}} V(r\theta)\dd{\mathsf{Unif}}(\theta)$ is $L_V$-smooth and $\ell_V$-strongly convex.
\end{lemma}
\begin{proof}
    We compute the first and second derivatives by two applications of the chain rule:
    \begin{align*}
        &\overline{V}'(r) = \frac{\dd}{\dd r} \int_{S^{d-1}} V(r\theta)\dd{\mathsf{Unif}}(\theta) =   \int_{S^{d-1}} \langle \nabla V(r\theta), \theta \rangle \dd {\mathsf{Unif}}(\theta)\,, \\
        &\overline{V}''(r) = \frac{\dd}{\dd r}\int_{S^{d-1}} \langle \nabla V(r\theta), \theta \rangle \dd {\mathsf{Unif}}(\theta) = \int_{S^{d-1}} \langle \theta, \nabla^2 V(r\theta)\, \theta \rangle \dd {\mathsf{Unif}}(\theta)\,.
    \end{align*}
As $0\preceq \ell_V I \preceq \nabla^2 V \preceq L_V I$, the proof is complete (where we use that $\int_{S^{d-1}} \|\theta\|^2_2 \dd {\mathsf{Unif}}(\theta) = 1$).
\end{proof}

We now compute the Hessian of $x \mapsto \overline{V}(\|x\|)$:
\begin{align*}
    \nabla \overline{V}(\|x\|) = \overline{V}'(\|x\|)\, \frac{x}{\|x\|}\,, \qquad \nabla\Bigr(\overline{V}'(\|x\|)\, \frac{x}{\|x\|}\Bigr) = \overline{V}''(\|x\|)\, \frac{xx^\top}{\|x\|^2} + \frac{\overline{V}'(\|x\|)}{\|x\|}\,\Bigl( I - \frac{xx^\top}{\|x\|^2}\Bigr)\,.
\end{align*}
By a rotation argument, we see that the Hessian can essentially be viewed as the following $d\times d$ diagonal matrix:
\begin{align*}
    \text{diag}(\overline{V}''(\|x\|), \overline{V}'(\|x\|)/\|x\|, \ldots,\overline{V}'(\|x\|)/\|x\|)\,.
\end{align*}
By the fundamental theorem of calculus and since $\overline V'(0) = 0$ (as a consequence of $\nabla V(0) = 0$), we see that 
\begin{align*}
    \ell_V \leq \overline{V}'(r)/r = (1/r) \int_0^r \overline{V}''(s)\dd s \leq L_V\,,
\end{align*}
which concludes the proof. 

\section{Proof of Theorem~\ref{theorem_map_regularity}}\label{app:reg_ot_map}

Let $T^\star_{\rm rad}$ denote the optimal transport map from $\rho = \cN(0,I)$ to $\pirad^\star$. By design, this map should be of the form
\begin{align*}
    T_{\rm rad}^\star(x) = \Psi^\star(\|x\|)\,x/\|x\| \defeq \psi^\star(\|x\|)\,x\,,
\end{align*}
where $\psi$ is some continuous, strictly monotone function.

As $\nabla^2(-\log\rho) = I$ and by Proposition~\ref{reg_radial}, it holds by a two-sided version of Caffarelli's contraction theorem \citep[see, e.g.,][]{chewi2023entropic} that
\begin{align}\label{eq:caff_proof}
    \frac{1}{\sqrt{L_V}\,} I \preceq D T_{\rm rad}^\star(x) \preceq \frac{1}{\sqrt{\ell_V}}\,I\,.
\end{align}
Computing $DT_{\rm rad}^\star$, we obtain
\begin{align*}
    DT_{\rm rad}^\star(x) = (\Psi^\star)'(\|x\|)\, \frac{xx^\top}{\|x\|^2} + \psi^\star(\|x\|)\,(I - xx^\top/\|x\|^2)\,.
\end{align*}
Combined with \eqref{eq:caff_proof}, this implies that
\begin{align}\label{eq:psi_inequalities}
\begin{split}
    &\frac{1}{\sqrt{L_V}} \leq (\Psi^\star)'(\|x\|) \leq \frac{1}{\sqrt{\ell_V}}\,, \\
    &\frac{1}{\sqrt{L_V}} \leq \psi^\star(\|x\|) \leq \frac{1}{\sqrt{\ell_V}}\,,
\end{split}
\end{align}
and yields the first claim.

For the second, we use the log-Monge--Amp\`ere equation between $\rho$ and $\pirad^\star$:
\begin{align*}
    \log\rho(x) = \log \pirad^\star(T_{\rm rad}^\star(x)) + \log\det(DT_{\rm rad}^\star(x))\,.
\end{align*}
Plugging in the expressions for $\rho$, $T_{\rm rad}^\star$, and $\pirad^\star$, one obtains
\begin{align}\label{eq:log_monge}
    \overline V(\Psi^\star(\|x\|)) = \frac{\|x\|^2}{2} + (d-1)\log(\psi^\star(\|x\|)) + \log((\Psi^\star)'(\|x\|)) + \log c\,,
\end{align}
where $c$ consists of normalizing constants. From here, we set $r \defeq \|x\|$ and rewrite~\eqref{eq:log_monge}, omitting the argument of $\Psi^\star$ for simplicity:
\begin{align}\label{eq:log_monge_r}
    \overline V\circ \Psi^\star = \frac{r^2}{2} - (d-1)\log r + (d-1)\log \Psi^\star + \log {(\Psi^\star)'} + \log c\,.
\end{align}
Differentiating in $r$, one obtains
\begin{align*}
    \frac{(\Psi^\star)''}{(\Psi^\star)'} = -r + \frac{d-1}{r} - (d-1)\,\frac{(\Psi^\star)'}{\Psi^\star} + \overline{V}'(\Psi^\star)\,(\Psi^\star)'\,. 
\end{align*}

Define the two functions
\begin{align*}
    F_\rho(r) \defeq \frac{r^2}{2} - (d-1) \log r\,, \qquad F_\star(r) \defeq \overline V(r) - (d-1)\log r\,.
\end{align*}
Let $r_\rho$, $r_\star$ denote the minimizers of $F_\rho$ and $F_\star$ respectively; thus,
\begin{align*}
    r_\rho
    &= \sqrt{d-1}\,, \qquad \overline V'(r_\star) - \frac{d-1}{r_\star} = 0\,.
\end{align*}
The intuition is that under $\rho$, the norm is concentrated around $r_\rho$, and similarly, under $\pirad^\star$, the norm is concentrated around $r_\star$.
We therefore expand~\eqref{eq:log_monge_r} around $r \approx r_\rho$, $\Psi^\star(r) \approx r_\star$.
We start by noting that
\begin{align*}
    F_\rho''(r)
    &= 1 + \frac{d-1}{r^2}\,, \qquad F_\star''(r) = \overline V''(r) + \frac{d-1}{r^2}\,.
\end{align*}
Taylor expansion, together with $F_\rho'(r_\rho) = 0$, shows that
\begin{align}\label{eq:F_deriv_bd1}
    |F_\rho'(r)|
    &= \Bigl\lvert\int_{r_\rho}^r F_\rho''(s) \dd s\Bigr\vert
    \le \Bigl(1 + \frac{d-1}{(r \wedge r_\rho)^2}\Bigr)\,|r-r_\rho|
    \lesssim \bigl(1 + \frac{d}{r^2}\bigr)\,|r-r_\rho|\,.
\end{align}
A similar argument yields
\begin{align*}
    |F_\star'(r)|
    &= \Bigl\lvert \int_{r_\star}^r F_\star''(s)\dd s\Bigr\rvert
    \le \Bigl(L_V + \frac{d-1}{(r \wedge r_\star)^2}\Bigr)\,|r-r_\star|\,.
\end{align*}
To simplify, we use
\begin{align*}
    r_\star
    &= \frac{d-1}{\overline V'(r_\star)}
    \ge \frac{d-1}{L_V r_\star}\,,
\end{align*}
so $r_\star^2 \ge (d-1)/L_V$. Therefore,
\begin{align}\label{eq:F_deriv_bd2}
    |F_\star'(r)|
    &\lesssim \bigl(L_V + \frac{d}{r^2}\bigr)\,|r-r_\star|\,.
\end{align}
Substituting~\eqref{eq:F_deriv_bd1} and~\eqref{eq:F_deriv_bd2} into~\eqref{eq:log_monge_r}, we see that
\begin{align}\label{eq:intermed_calc}
    \frac{|(\Psi^\star)''|}{(\Psi^\star)'}
    &\le |F_\rho'(r)| + |(\Psi^\star)'|\,|F_\star'(\Psi^\star)|
    \lesssim \bigl(1 + \frac{d}{r^2}\bigr)\,|r-r_\rho| + |(\Psi^\star)'|\, \bigl(L_V + \frac{d}{(\Psi^\star)^2}\bigr)\,|\Psi^\star - r_\star|\,.
\end{align}
The next step is to show that $\Psi^\star(r_\rho) \approx r_\star$, i.e., that the map $\Psi^\star$ approximately pushes forward the mode of the radial part of $\rho$ to the radial part of $\pirad^\star$.

\begin{lemma}
    \begin{align*}
        |\Psi^\star(r_\rho) - r_\star|
        &\le \frac{2}{\sqrt{\ell_V}}\,.
    \end{align*}
\end{lemma}
\begin{proof}
    We use the fact that for any $\alpha$-strongly log-concave distribution $\pi$ in dimension $d$ with mode $x_\star$, it holds that $\int \|x-x_\star\|^2\dd \pi(x) \le d/\alpha$~\citep[c.f.][``basic lemma'']{ChewiBook}.
    Let $X_\rho \sim \rho$, so that $\Psi^\star(\|X_\rho\|)$ is distributed according to the radial part of $\pirad^\star$.
    The radial parts of $\rho$ and of $\pirad^\star$ are one-dimensional distributions, with potentials $F_\rho$ and $F_\star$ respectively; this implies that they are both strongly log-concave, with respective parameters $1$ and $\ell_V$.
    Therefore,
    \begin{align*}
        |\Psi^\star(r_\rho) - r_\star|
        &\le \E|\Psi^\star(r_\rho) - \Psi^\star(\|X_\rho\|)| + \E|\Psi^\star(\|X_\rho\|) - r_\star|\\
        &\le \frac{1}{\sqrt{\ell_V}}\,\E\bigl\lvert\,r_\rho - \|X_\rho\|\,\bigr\rvert + \E|\Psi^\star(\|X_\rho\|) - r_\star| \\
        &\le \frac{1}{\sqrt{\ell_V}} + \frac{1}{\sqrt{\ell_V}}
        = \frac{2}{\sqrt{\ell_V}}\,. \qedhere
    \end{align*}
\end{proof}

Continuing from~\eqref{eq:intermed_calc}, since $\Psi^\star \ge r/\sqrt{L_V}$,
\begin{align*}
    \frac{|(\Psi^\star)''|}{(\Psi^\star)'}
    &\lesssim \bigl(1 + \frac{d}{r^2}\bigr)\,|r-r_\rho| + L_V\, |(\Psi^\star)'|\, \bigl(1 + \frac{d}{r^2}\bigr)\,\bigl(|\Psi^\star - \Psi^\star(r_\rho)| + |\Psi^\star(r_\rho) - r_\star|\bigr) \\
    &\lesssim \bigl(1 + \frac{d}{r^2}\bigr)\,\Bigl(|r-r_\rho| + L_V\,|(\Psi^\star)'|\,\bigl(\frac{|r-r_\rho|}{\sqrt{\ell_V}} + \frac{1}{\sqrt{\ell_V}}\bigr)\Bigr)
    \lesssim \kappa\,\bigl(1 + \frac{d}{r^2}\bigr)\,(1+|r-r_\rho|)\,.
\end{align*}
This proves the estimate
\begin{align*}
    |(\Psi^\star)''|
    &\lesssim \frac{\kappa}{\sqrt{\ell_V}}\,\bigl(1 + \frac{d}{r^2}\bigr)\,(1+|r-r_\rho|)
    \lesssim \frac{\kappa}{\sqrt{\ell_V}}\,\bigl(1 + \frac{d}{r^2}\bigr)\,(1+|r-\sqrt d|)\,.
\end{align*}

\section{Omitted proofs from Section~\ref{sec:comp_guarantees}}\label{app:comp_guarantees}
\subsection{Proof of Theorem~\ref{theorem_unif_approx}}\label{app:unif_approx_proof}
Our goal is to find a map $\hat{T} = T_{\hat\lambda}$ such that
\begin{equation}\label{desired_bound}
    \|T^\star - \hat{T}\|_{L^{2}(\rho)}^2 \leq \frac{\epsilon^2}{\ell_V} \qquad \text{and} \qquad \|D(T^\star - \hat{T})\|_{L^{2}(\rho)}^2 \leq \frac{\epsilon_1^2}{\ell_V}\,.
\end{equation}
We will closely follow the proof strategy put forth by \cite{JiaChePoo24MFVI} with appropriate modifications throughout.

\textbf{Step 1: Reformulate the problem in terms of $\Psi$.}
Let us write $\hat\Psi$ for the radial part of $\hat T$, i.e., $\|\hat T(x)\| = \hat\Psi(\|x\|)$.
Immediately, we notice the following simplification can be made:
\begin{align*}
    \|T^\star - \hat{T}\|_{L^{2}(\rho)}^2 = \int \|T^\star(x) - \hat{T}(x)\|^2 \dd \rho(x)= \int (\Psi^\star(\|x\|) - \hat\Psi(\|x\|))^2\dd \rho(x)= \|\Psi^\star - \hat\Psi\|^2_{L^2(\tilde\rho)}\,,
\end{align*}
where $\tilde\rho \defeq {\rm{Law}}(\|X\|)$ for $X \sim \rho = \cN(0,I)$. 
Similarly,
\begin{align*}
    \|D(T^\star - \hat T)\|_{L^2(\rho)}^2
    &= \int \Bigl\lVert (\Psi^\star - \hat\Psi)'(\|x\|)\,\frac{xx^\top}{\|x\|^2} + (\psi^\star - \hat\psi)(\|x\|)\,\bigl(I - \frac{xx^\top}{\|x\|^2}\bigr)\Bigr\rVert_{\rm F}^2\dd \rho(x) \\
    &\lesssim \int \bigl\{ |(\Psi^\star - \hat\Psi)'(\|x\|)|^2 + d\,(\psi^\star - \hat\psi)^2(\|x\|) \bigr\}\dd \rho(x) \\
    &\lesssim \int \bigl\{ |(\Psi^\star - \hat\Psi)'(r)|^2 + d\,(\psi^\star - \hat\psi)^2(r) \bigr\}\dd \tilde\rho(r)
    = \|(\Psi^\star - \hat\Psi)'\|_{L^2(\tilde\rho)}^2 + d\,\|\psi^\star - \hat\psi\|_{L^2(\tilde\rho)}^2\,.
\end{align*}
From~\eqref{eq:psi_inequalities}, we know that $1/\sqrt{L_V} \le (\Psi^\star)' \le 1/\sqrt{\ell_V}$.

\textbf{Step 2: Remove the strongly convex part.}
Recall that by definition of $\cT_J$ and $\alpha = 1/\sqrt{L_V}$,
\begin{align*}
    \hat\Psi(r) = \frac{r}{\sqrt{L_V}} + \underbrace{\sum_{j=0}^J \hat\lambda_j \Psi_j(r)}_{\eqqcolon \hat\Psi_\diamond(r)}\,.
\end{align*}
Let us also write $\Psi^\star(r) \defeq r/\sqrt{L_V} + \Psi^\star_\diamond(r)$.
Then, $\|\Psi^\star - \hat\Psi\|_{L^2(\tilde\rho)}^2 = \|\Psi^\star_\diamond - \hat\Psi_\diamond\|_{L^2(\tilde\rho)}^2$, etc.
Hence, our goal is equivalent to approximating $\Psi_\diamond^\star$ using a function of the form $\sum_{j=0}^J \hat\lambda_j \Psi_j$, where we know that $\Psi^\star_\diamond$ satisfies the bounds $0 \le (\Psi_\diamond^\star)' \le 1/\sqrt{\ell_V} - 1/\sqrt{L_V}$.
For simplicity, we will replace the upper bound on $(\Psi^\star_\diamond)'$ by $1/\sqrt{\ell_V}$, which only makes our problem more difficult.
Having reformulated our goal, we simply write $\hat\Psi$ for $\hat\Psi_\diamond$ and $\Psi^\star$ for $\Psi^\star_\diamond$ in the remaining steps to keep the notation concise.

\textbf{Step 3: Truncate.}
Next, we consider the interval $\cI \defeq [\sqrt{d}-R, \sqrt{d} + R]$, for some $R > 0$ that will be chosen later, and equally partition said interval into subintervals of length $\delta>0$. This defines the dictionary $\{\Psi_j\}_{j=0}^J$. Next, our construction will ensure that $\hat{\Psi}(\sqrt{d}-R) = \Psi^\star(\sqrt{d}-R)$ and similarly $\hat{\Psi}(\sqrt{d}+R) = \Psi^\star(\sqrt{d}+R)$---outside this interval, the function $\hat\Psi$ is a constant.
We want to bound
\begin{align*}
    \|\Psi^\star - \hat\Psi\|^2_{L^2(\tilde\rho)}  = (\mathsf{T1}) + (\mathsf{T2})
    \defeq \int_{\cI} |\Psi^\star(r) - \hat\Psi(r)|^2 \dd \tilde\rho(r) + \int_{\R \setminus \cI} |\Psi^\star(r) - \hat\Psi(r)|^2 \dd \tilde\rho(r)\,.
\end{align*}

We require the following lemma.
\begin{lemma}\label{lem:helper_lemma}
For $r \notin \cI$, then
\begin{align*}
    |\Psi^\star(r) - \hat\Psi(r)| \leq |r-\sqrt d|/\sqrt{\ell_V}\,.
\end{align*}
\end{lemma}
\begin{proof}
First take $r \geq \sqrt{d} + R$. Then, as $\Psi^\star, \hat\Psi \geq 0$ and by~\eqref{eq:psi_inequalities},
\begin{align*}
    |\Psi^\star(r) - \hat\Psi(r)|
    &= |\Psi^\star(r) - \Psi^\star(\sqrt d + R)|
    = \Psi^\star(r) - \Psi^\star(\sqrt d + R)
    \le (r-\sqrt d - R)/\sqrt{\ell_V}
    \le (r-\sqrt d)/\sqrt{\ell_V}\,.
\end{align*}

A symmetric argument completes the claim.
\end{proof}
Now, we bound via Cauchy--Schwarz,
\begin{align*}
    (\mathsf{T2}) = \int_{\R\setminus \cI}|\Psi^\star(r) - \hat\Psi(r)|^2\dd\tilde\rho(r) \leq \ell_V^{-1}\int_{\R\setminus\cI} (r-\sqrt d)^2 \dd\tilde\rho(r) \leq \ell_V^{-1} \sqrt{\E\bigl[(\|X\|-\sqrt d)^2\bigr]}\,\mathbb{P}(\|X\|\notin \cI)
\end{align*}
where $X\sim\rho$. Using standard tools from high-dimensional probability pertaining to Gaussian tail bounds \citep{van2014probability}, the right-hand side is bounded by $O(\ell_V^{-1}\exp(-R^2/2))$.
Therefore, if $R\gtrsim \sqrt{\log(1/\epsilon)}$, then $(\mathsf{T1}) \leq \epsilon^2/\ell_V$.

For $r \notin \cI$, $\hat\Psi'(r) = 0$, and $|(\Psi^\star)'| \le 1/\sqrt{\ell_V}$.
Thus,
\begin{align*}
    \int_{\R\setminus \cI} |(\Psi^\star - \hat\Psi)'|^2\dd\tilde\rho
    \le \frac{1}{\ell_V}\,\mathbb P(\|X\| \notin \cI)
    \lesssim \frac{\exp(-R^2/2)}{\ell_V}\,.
\end{align*}
Next, for $r \le \sqrt d - R$, {$\hat\Psi(r) = \Psi^\star(\sqrt d - R)$, and since $\psi^\star \le 1/\sqrt{\ell_V}$ everywhere,
\begin{align*}
    d\int_0^{\sqrt d - R} (\psi^\star-\hat\psi)^2
    &\lesssim d\int_0^{\sqrt d -R} \bigl(\frac{\Psi^\star(\sqrt d - R)^2}{r^2} + \frac{1}{\ell_V}\bigr)\dd \tilde \rho(r) \\
    &\le d\Psi^\star(\sqrt d - R)^2\, \E[\norm X^{-2}\,\mathbf 1_{\norm X \notin \cI}] + \frac{d}{\ell_V}\,\mathbb P(\norm X \notin \cI) \\
    &\le d\Psi^\star(\sqrt d - R)^2\,\E[\norm X^{-2.5}]^{4/5}\,\mathbb P(\norm X \notin \cI)^{1/5} + \frac{d}{\ell_V}\,\mathbb P(\norm X \notin \cI) \\
    &\lesssim d\, \frac{\Psi^\star(\sqrt d - R)^2}{d}\,\mathbb P(\norm X \notin \cI)^{1/5} + \frac{d}{\ell_V}\,\mathbb P(\norm X \notin \cI) \\
    &\lesssim \frac{d}{\ell_V} \exp(-R^2/10)\,,
\end{align*}
where we used $d \ge 3$.
A similar argument controls the integral over $[\sqrt d + R,\infty)$.
}
Hence, $R \gtrsim\sqrt{\log(d/\eps)}$ implies $\int_{\|x\|\notin \cI} \| D(T^\star - \hat T)(x)\|_{\rm F}^2\dd \rho(x) \lesssim \epsilon/\ell_V$.

\textbf{Step 4: Control the approximation error on $\cI$.}
Recall that our construction should satisfy $\hat\Psi(\sqrt d-R) = \Psi^\star(\sqrt d-R)$ and $\hat\Psi(\sqrt d + R) = \Psi^\star(\sqrt d + R)$.
Since $\hat\Psi(\sqrt d - R) = \hat\lambda_0 \Psi_0(\sqrt d - R) = \hat\lambda_0$, the first condition amounts to setting $\hat\lambda_0 = \Psi^\star(\sqrt d-R)$.
For the second condition, we will in fact ensure that $\hat\Psi$ agrees with $\Psi^\star$ at the endpoints of every sub-interval $[a,a+\delta]$, for each knot $a$.

We now turn toward bounding $(\mathsf{T1})$. To do so, consider a sub-interval $[a,a+\delta]$ on which $\hat\Psi$ is affine.
Since the graph of $\hat\Psi$ interpolates the points $(a, \Psi^{\star}(a))$ and $(a+\delta, \Psi^{\star}(a+\delta))$,
\begin{align*}
    \hat{\Psi}(r)= \Psi^{\star}(a) + \delta^{-1} \left(\Psi^{\star}(a+\delta) - \Psi^{\star}(a) \right)(r-a)\,.
\end{align*}
On the other hand, by two applications of the mean value theorem, we have
\begin{align*}
    \Psi^{\star}(r) = \Psi^{\star}(a) + (\Psi^{\star})'(c_1)\,(r-a)\,, \quad \Psi^\star(a+\delta) = \Psi^\star(a) + (\Psi^\star)'(c_2)\,\delta\,,
\end{align*}
for $c_1,c_2 \in [a,a+\delta]$. Thus,
\begin{align*}
    |\Psi^\star(r) - \hat\Psi(r)| = |((\Psi^\star)'(c_1) - (\Psi^\star)'(c_2))\,(r-a)|\,.
\end{align*}
Applying the mean value theorem again (as $\Psi^\star$ has two bounded derivatives), we have that
\begin{align*}
    (\Psi^\star)'(c_1) - (\Psi^\star)'(c_2) = (\Psi^\star)''(c_3)\,(c_1-c_2)\,,
\end{align*}
for some $c_3 \in [c_1,c_2]$, and thus we have a bound
\begin{align*}
    |\Psi^\star(r) - \hat\Psi(r)| = |(\Psi^\star)''(c_3)|\,|c_1-c_2|\,|r-a|\leq |(\Psi^\star)''(c_3)|\, \delta^2\,.
\end{align*}
Using the gradient bound on $(\Psi^\star)''$ from Theorem~\ref{theorem_map_regularity}, we obtain
\begin{align*}
    |\Psi^\star(r) - \hat\Psi(r)| \lesssim \frac{\kappa\delta^2}{\sqrt{\ell_V}}\sup_{\xi\in[a,a+\delta]}\bigl(1 + \frac{d}{\xi^2}\bigr)\,(1+|\xi-\sqrt d|)
    \lesssim \frac{\kappa \delta^2 R}{\sqrt{\ell_V}}\,,
\end{align*}
provided $R \ll \sqrt d$.
And so, over all of $\cI$, this becomes
\begin{align*}
    \sup_{r \in \cI}|\Psi^\star(r) - \hat\Psi(r)| \lesssim \frac{\kappa\delta^2 R}{\sqrt{\ell_V}}\,.
\end{align*}

We now prove the uniform bound for the gradient difference $D(T^\star - \hat T)$.
Since $\hat\Psi'(r) = (\Psi^\star(a+\delta) - \Psi^\star(a))/\delta = (\Psi^\star)'(c_2)$, it follows that
\begin{align*}
    |(\Psi^\star - \hat\Psi)'(r)|
    &= |(\Psi^\star)'(r) - (\Psi^\star)'(c_2)|
    \lesssim \frac{\kappa \delta R}{\sqrt{\ell_V}}\,.
\end{align*}
Similarly,
\begin{align*}
    |\psi^\star(r) - \hat\psi(r)|
    &= \frac{|\Psi^\star(r) - \hat\Psi(r)|}{r}
    \lesssim \frac{\kappa \delta^2 R}{\sqrt{d\ell_V}}\,.
\end{align*}
Consequently, if we choose $\delta \asymp \epsilon^{1/2}/(\kappa^{1/2} R^{1/2})$, then we can ensure
\begin{align*}
    \int_{\|x\|\in \cI} \|(T^\star - \hat T)(x)\|^2\dd\rho(x)
    &\lesssim \frac{\epsilon^2}{\ell_V}\,, \qquad  \int_{\|x\|\in \cI} \|D(T^\star - \hat T)(x)\|_{\rm F}^2\dd\rho(x) \lesssim \frac{\kappa R \epsilon}{\ell_V}\,.
\end{align*}

\textbf{Step 5: Complete the proof.}
Combining all of the bounds, we set $R \asymp \sqrt{\log(d/\epsilon)}$, $\delta \asymp \sqrt{\epsilon/(\kappa \log(d/\epsilon))}$.
Since we require $R \ll \sqrt d$, this requires $d \gg \log(1/\eps)$.
With these choices, we can ensure
\begin{align*}
    \|T^\star - \hat T\|_{L^2(\rho)}^2 \le \frac{\epsilon^2}{\ell_V}\,, \qquad \|D(T^\star - \hat T)\|_{L^2(\rho)}^2 \lesssim \frac{\kappa \epsilon \sqrt{\log(d/\epsilon)}}{\ell_V}\,.
\end{align*}
This completes the proof. 

\subsection{Construction of the Gram matrix}\label{app:gram_construction}
In this section, we describe how to compute the Gram matrix $Q \in \mathbb{S}_+^{J+1}$ explicitly, where we recall that
\begin{align*}
    Q_{i,j} \defeq \E_{X \sim \rho}[\Psi_i(\|X\|)\,\Psi_j(\|X\|)] = \int_0^\infty \Psi_i(r)\,\Psi_j(r) \dd \tilde{\rho}(r)\,,
\end{align*}
where 
\begin{align} \label{eq:chi_density}
    \dd \tilde{\rho}(r) = \frac{1}{2^{d/2-1}\Gamma(d/2)} r^{d-1} e^{-r^2/2} \dd r\,,
\end{align}
which is the distribution of $r = \|X\|$ for $X \sim \cN(0,I)$ (also known as the chi distribution). Indeed, as as we'll see below, the piecewise linear nature of $\{\Psi_j\}_{j=0}^J$ will allow us to compute $Q_{i,j}$. To this end, we also require the following object, the $n$-th truncated moment of $\tilde \rho$ over an interval $[a, b]$:
\begin{align} \label{eq:moments}
    \mathcal{M}_n(a, b) &\defeq \int_a^b r^n \dd \tilde{\rho}(r) \nonumber = \frac{2^{n/2}}{\Gamma(d/2)} \left[ \Gamma\left(\frac{n+d}{2}, \frac{a^2}{2}\right) - \Gamma\left(\frac{n+d}{2}, \frac{b^2}{2}\right) \right]\,,
\end{align}
where $\Gamma(s, x)$ is the upper incomplete gamma function.

To compute $Q_{i,j}$ for $i \le j$ (assuming $a_i \le a_j$), we decompose the integral based on the support of the basis functions. Namely, we consider a small sub-interval $\left[a_i, a_i+\delta\right]$ for $\Psi_i$ and $\left[a_j, a_j+\delta\right]$ for $\Psi_j$.
\paragraph{Case 1: Disjoint Ramps ($a_j \ge a_i + \delta$).}
The rising part of $\Psi_j$ starts after $\Psi_i$ has already reached its plateau of 1. The integral splits into the ramp region of $\Psi_j$ and the region where both functions are unity:
\begin{align*}
    Q_{i,j} &= \int_{a_j}^{a_j+\delta} (1) \cdot \left(\frac{r-a_j}{\delta}\right) \dd \tilde{\rho}(r) + \int_{a_j+\delta}^{\infty} (1) \cdot (1) \dd \tilde{\rho}(r) \\
    &= \frac{1}{\delta}\left[\mathcal{M}_1(a_j, a_j+\delta) - a_j \mathcal{M}_0(a_j, a_j+\delta)\right] + \mathcal{M}_0(a_j+\delta, \infty)\,.
\end{align*}

\paragraph{Case 2: Overlapping Ramps ($a_j < a_i + \delta$).}
    In this setting, the domain of integration consists of three non-zero overlapping regions: (1) where both are ramps, (2) where $\Psi_i$ saturates but $\Psi_j$ is still rising, and (3) where both saturate:
\begin{equation*}
    Q_{i,j} = \int_{a_j}^{a_i+\delta} \left(\frac{r-a_i}{\delta}\right)\left(\frac{r-a_j}{\delta}\right) \dd \tilde{\rho}(r)+ \int_{a_i+\delta}^{a_j+\delta} (1) \cdot \left(\frac{r-a_j}{\delta}\right) \dd \tilde{\rho}(r) + \int_{a_j+\delta}^{\infty} (1) \cdot (1) \dd \tilde{\rho}(r)\,.
\end{equation*}
Expanding the quadratic term in the first integral allows the entire expression to be computed as a sum of weighted moments $\mathcal{M}_n$. 

\subsection{Proof of Proposition~\ref{prop:cf_smooth_convex}}\label{app:proof_smooth_cvx}

We follow the arguments of~\citet{JiaChePoo24MFVI}. By~\citet[Propositions 5.10 and 5.11]{JiaChePoo24MFVI}, the objective $\lambda \mapsto \cF(T_\lambda)$ is $\ell_V$-strongly convex and $L_V\,(1+\Upsilon)$ smooth in the $\|\cdot\|_Q$ norm, where $\Upsilon > 0$ is the smallest positive constant such that $\|DT_\lambda-  \alpha I\|_{L^2(\rho)}^2 \le \Upsilon \,\|T_\lambda - \alpha \,{\id}\|_{L^2(\rho)}^2$ for all $\lambda \in\R^{J+1}$.
Writing this out, we need
\begin{align*}
    &\int \Bigl\lVert \sum_{j=0}^J \lambda_j \,\Bigl(\Psi_j'(\|x\|)\,\frac{xx^\top}{\|x\|^2} + \psi_j(\|x\|)\,\bigl(I - \frac{xx^\top}{\|x\|^2}\Bigr)\Bigr\rVert_{\rm F}^2 \dd \rho(x) \\
    &\qquad = \int \Bigl\{\Bigl(\sum_{j=0}^J \lambda_j \Psi_j'(\|x\|)\Bigr)^2 + (d-1)\,\Bigl(\sum_{j=0}^J \lambda_j \psi_j(\|x\|)\Bigr)^2\Bigr\}\dd\rho(x) \\
    &\qquad = \int \Bigl\{\Bigl(\sum_{j=0}^J \lambda_j \Psi_j'(r)\Bigr)^2 + (d-1)\,\Bigl(\sum_{j=0}^J \lambda_j \psi_j(r)\Bigr)^2\Bigr\}\dd\tilde\rho(r)
\end{align*}
to be bounded by
\begin{align*}
    \Upsilon\int \Bigl(\sum_{j=0}^J \lambda_j \Psi_j(r)\Bigr)^2\dd\tilde\rho(r)\,.
\end{align*}
It suffices to verify
\begin{align*}
    \int_{a_k}^{a_k+\delta_k} \Bigl\{\Bigl(\sum_{j=0}^J \lambda_j \Psi_j'(r)\Bigr)^2 + (d-1)\,\Bigl(\sum_{j=0}^J \lambda_j \psi_j(r)\Bigr)^2\Bigr\}\dd\tilde\rho(r) \le 
    \Upsilon\int_{a_k}^{a_k+\delta_k} \Bigl(\sum_{j=0}^J \lambda_j \Psi_j(r)\Bigr)^2\dd\tilde\rho(r)
\end{align*}
for each $k =0,1,\dotsc,J$ separately.
If we let $\bar \Psi_\lambda(\|x\|) \defeq \|T_\lambda(x) - \alpha x\|$, this reduces to
\begin{align*}
    \int_{a_k}^{a_k+\delta_k} \bigl\{\underbrace{\lambda_k^2 \delta_k^{-2}}_{\mathsf T_1} + \underbrace{\frac{d-1}{r^2}\,\bigl(\bar \Psi_\lambda(a_k) + \lambda_k \delta_k^{-1}\,(r-a_k)\bigr)^2}_{\mathsf T_2}\bigr\}\dd\tilde\rho(r) \le 
    \Upsilon\int_{a_k}^{a_k+\delta_k} \bigl(\bar \Psi_\lambda(a_k) + \lambda_k \delta_k^{-1}\,(r-a_k)\bigr)^2\dd\tilde\rho(r)\,.
\end{align*}
Let us start with the case of $k \ge 1$.
In this case, $r\gtrsim \sqrt d$, so the term $\mathsf T_2$ is bounded, up to an absolute constant, by the integrand on the right-hand side.
Thus, for term $\mathsf T_1$, it suffices to prove $\int_{a_k}^{a_k+\delta} \dd\tilde\rho(r) \lesssim \Upsilon \inf_{\bar r\in \R} \int_{a_k}^{a_k+\delta} (r-\bar r)^2\dd\tilde\rho(r)$, i.e., we need a lower bound on the variance of the distribution $\tilde\rho$ restricted to the interval $[a_k, a_k+\delta]$.
Note that $\tilde\rho \propto \exp(-F_\rho)$ where $F_\rho$ is defined in the proof of Theorem~\ref{theorem_map_regularity}.
Using the estimates from that proof, we know that $F_\rho$ is $O(R)$-Lipschitz on the interval $[a_k,a_k+\delta]$, hence $\log\tilde\rho$ only varies by $O(1)$ on this interval provided $\delta \lesssim 1/R$.
The argument of~\citet[Lemma 5.13]{JiaChePoo24MFVI} now shows that the variance is lower bounded by $\Omega(\delta^2)$, and hence our desired estimate holds for $\Upsilon \asymp \delta^{-2}$.

Next, consider the case $k=0$, so that $a_0 = 0$ and $\delta_0 = \sqrt d - R$.
For the term $\mathsf T_2$, we must prove $(d-1) \int_0^{\delta_0} \dd\tilde\rho(r) \lesssim \Upsilon \int_0^{\delta_0} r^2 \dd\tilde\rho(r)$, which holds with $\Upsilon \asymp 1$: indeed, this follows from the fact that $\tilde\rho|_{[0,\delta_0]}$ is $1$-strongly log-concave with mode at $\delta_0 \asymp \sqrt d$, so $\int (r-\delta_0)^2\dd\tilde\rho|_{[0,\delta_0]}(r)\le 1$.
The term $\mathsf T_1$ is similar but easier.

All in all, this shows that we can take $\Upsilon \lesssim \delta^{-2} = \widetilde\Theta(J^2)$.

\subsection{Proof of Theorem~\ref{thm:radvi_converges}}\label{app:radvi_converges_proof}
Recall that for $\lambda \mapsto \cF(T_\lambda)$, the smoothness constant is $\widetilde O(J^2 L_V)$. Choosing $h = 1/\widetilde\Theta(L_V J^2)$, we have by standard arguments for projected gradient descent for smooth and strongly convex functions \citep[Theorem 10.29]{Beck2017} 
\begin{align*}
    \|T_{\lambda^{(k)}} - T_{\lambda^\star}\|^2_{L^2(\rho)} \defeq \|\lambda^{(k)} - \lambda^{\star}\|^2_Q \leq \epsilon^2/\ell_V
\end{align*}
so long as $k \gtrsim_{\log} \kappa J^2 \log({\kl{(T_{\lambda^{(0)}})_\sharp \rho}{\pi}}/\epsilon^2)$. {After two applications of triangle inequality, and invoking \eqref{eq:helper_inequality}, we arrive at
\begin{align*}
    \|T_{\lambda^{(k)}} - T^\star_{\rm rad}\|^2_{L^2(\rho)} &\lesssim \|T_{\lambda^{(k)}} - T_{\lambda^{\star}}\|^2_{L^2(\rho)} + \|T_{\lambda^{\star}} - T_{\hat \lambda}\|^2_{L^2(\rho)} + \|T_{\hat\lambda} - T^\star_{\rm rad}\|^2_{L^2(\rho)}\\
    &\lesssim\|\lambda^{(k)} - \lambda^{\star}\|^2_Q + \kappa\, \|T_{\hat \lambda} - T_{\rm rad}^\star\|^2_{L^2(\rho)} + \kappa^2\, \|D(T_{\hat \lambda} - T_{\rm rad}^\star)\|^2_{L^2(\rho)} \\
    &\lesssim \eps^2/\ell_V +  \kappa\, \|T_{\hat \lambda} - T_{\rm rad}^\star\|^2_{L^2(\rho)} + \kappa^2\, \|D(T_{\hat \lambda} - T_{\rm rad}^\star)\|^2_{L^2(\rho)}\,.
\end{align*}
For $\eps_1 > 0$, Theorem~\ref{theorem_unif_approx} states that for $J = \tilde{\Omega}(\sqrt{\kappa/\eps_1})$, then we have
\begin{align*}
    \|T_{\lambda^{(k)}} - T^\star_{\rm rad}\|^2_{L^2(\rho)} \lesssim_{\rm log} \eps^2/\ell_V + \kappa \eps_1^2/\ell_V + \kappa^3 \eps_1/\ell_V\,.
\end{align*}
Choosing $\eps_1 \asymp_{\log} \eps^2/\kappa^3$ (in the worst case), we conclude that the error tolerance is $\eps^2/\ell_V$. Then, $J = \tilde{\Theta}(\kappa^2/\eps)$, we have that the number of iterations required scales like $\kappa^5\eps^{-1}$ up to logarithmic factors, and the required stepsize is $\eps/(L_V\kappa^2)$ up to logarithmic factors.}

\subsection{Gradient of the objective}\label{app:grad_objective}
In the stochastic projected gradient descent algorithm, we require the gradient with respect to $\lambda$ of
\[
\cF(T_\lambda)
=\int V(T_\lambda(x))\,\mathrm{d}\rho(x)
-\int \log\det DT_\lambda(x)\dd\rho(x)\,.
\]
Passing the gradient under the integral, we need to compute
\begin{equation}\label{eq:grad_kl}
    \nabla_{\lambda}\cF(T_\lambda)
    = \int \nabla_{\lambda} V(T_\lambda(x)) \dd \rho(x) - \int \nabla_{\lambda}\log\det DT_\lambda(x) \dd \rho(x)
\end{equation}
We first compute the gradient of the potential energy. Writing $r=\|x\|$ and letting $\{\Psi_j\}_{j=0}^J$ denote our basis, the chain rule gives for each coordinate $j\in\{0,1,\ldots,J\}$
\begin{align*}
\nabla_{\lambda_j} V(T_\lambda(x))
= \Psi_j(r) \langle x/r, \nabla V\big(T_\lambda(x)\big)\rangle.
\end{align*}
Thus, the full gradient becomes
\begin{equation}\label{eq:grad_norm_v}
\nabla_{\lambda} V(T_\lambda(x))
= \vec\Psi(r)\langle x/r, \nabla V(T_\lambda(x))\rangle ,
\end{equation}
where $\vec\Psi(r) = (\Psi_0(r),\ldots,\Psi_J(r)) \in \R^{J+1}_{+}$. For this term, we approximate the full gradient via Monte Carlo approximation, i.e.,
\begin{align*}
\nabla_\lambda \E_{X\sim\rho}[V(T_\lambda(X))] \simeq \frac{1}{n}\sum_{i=1}^n \vec\Psi(\|X_i\|)\langle X_i/\|X_i\|, \nabla V(T_\lambda(X_i))\rangle ,
\end{align*}
where $X_1,\ldots,X_n \sim \rho$.

We next compute the gradient of the log-determinant term. For any $\lambda \in \R^{J+1}_+$, one can compute
\begin{align*}
    DT_\lambda(x)
=\frac{\alpha r+\sum_{j=0}^J \lambda_j \Psi_j(r)}{r}\,I
+\Bigl(\sum_{j=0}^J \lambda_j \Psi_j'(r) - \frac{1}{r}\sum_{j=0}^J \lambda_j \Psi_j(r)\Bigr)\,\frac{xx^\top}{r^2}\,,
\end{align*}
and, upon rearranging
\begin{align*}
    DT_\lambda(x) = \Bigl(\alpha + \sum_{j=0}^J \lambda_j \Psi'_j(r)\Bigr)\,xx^\top/r^2+ \Bigl(\alpha + \sum_{j=0}^J {\lambda_j} \Psi_j(r)/r\Bigr)\,(I - xx^\top/r^2)\,.
\end{align*}
The eigenvalue in the radial direction $x/r$ equals
\[
\alpha+\sum_{j=0}^J \lambda_j \Psi_j'(r)\,,
\]
while each of the remainig $d-1$ orthogonal directions has eigenvalue
\[
\alpha +\sum_{j=0}^J \lambda_j \Psi_j(r)/r\,.
\]
Accounting for the multiplicity of these eigenvalues, one arrives at 
\begin{align*}
\log\det DT_\lambda(x)
=(d-1)\log\bigl(\alpha +\langle \lambda,\vec\Psi(r)/r\rangle\bigr)
+\log\bigl(\alpha+\langle \lambda,\vec\Psi'(r)\rangle\bigr)\,.
\end{align*}
Writing $g_\lambda(r) = \alpha r + \langle \lambda, \Psi(r)\rangle$ and thus {$g'_\lambda(r) = \alpha+ \langle \lambda, \Psi'(r)\rangle$, differentiating with respect to $\lambda$ yields
\begin{equation}\label{eq:grad_norm_h}
\nabla_{\lambda_j}\log\det DT_\lambda(x)
= \frac{(d-1)\,\Psi_j(r)}{g_\lambda(r)} + \frac{\Psi'_j(r)}{g'_\lambda(r)}\,.
\end{equation}}
From \eqref{eq:grad_norm_h} we have
\begin{align}\label{eq:gradlogdet_app}
\nabla_{\lambda_j}\,\mathbb E_{X \sim \rho}[\log\det DT_\lambda(X)]
=\int_0^\infty
\Bigl[\frac{(d-1)\,\Psi_j(r)}{g_\lambda(r)} + \frac{\Psi'_j(r)}{g'_\lambda(r)}\Bigr]\dd\tilde\rho(r)    
\end{align}
where $\tilde\rho$ is the radial law of $\|X\|$ under $\rho = \cN(0,I)$. There are two approaches to computing this integral. {One approach is via Monte Carlo. Below, we describe how to efficiently evaluate the integral deterministically for any $j \in \{0,\ldots,J\}$.}

To this end, define the sets
\begin{align*}
    I_0 \defeq [0,a_1], \quad I_\ell \defeq [a_\ell, a_\ell + \delta]
\end{align*}
for $\ell = 1,\ldots,J$, and let $I_{J+1} = \{r : r \geq a_J + \delta\}$. Thus, we can partition the support of $\tilde\rho$ as $[0,\infty) = \bigcup_{j=0}^{J+1}I_j$.

First, notice that $\Psi'_j = 1/\delta$  on the interval $\ell = j$ and zero otherwise. Furthermore, the derivative of all other basis functions evaluates to zero on this interval (they are only non-zero on their own interval). Thus, on the interval $I_j$, $g_\lambda'(r) = \alpha + \lambda_j \Psi'_j(r)$, and the second integral becomes
\begin{align*}
    \int_{I_j} \frac{1/\delta}{\alpha + \lambda_j/\delta}\dd\tilde\rho(r) = \frac{1}{\alpha\delta + \lambda_j}\mathbb{P}(\|X\| \in I_j)\,.
\end{align*}
This can be computed easily, and takes care of the second part of \eqref{eq:gradlogdet_app} for each component.

For the first part of \eqref{eq:gradlogdet_app}, notice that for a given $j$ index, the numerator of the integrand is zero on all $I_\ell$ for $\ell < j$. We arrive at
\begin{align*}
    (d-1)\sum_{\ell \geq j}\int_{I_\ell} \frac{\Psi_j(r)}{\alpha r + \sum_{k}\lambda_k\Psi_k(r)}\dd\tilde\rho(r)\,.
\end{align*}
However, in the denominator, for $k < \ell$, $\Psi_k = 1$ and, for $k > \ell$, $\Psi_k = 0$ and the integral simplifies to
\begin{align*}
    (d-1)\sum_{\ell \geq j}\int_{I_\ell} \frac{\Psi_j(r)}{\alpha r + \sum_{k < \ell}\lambda_k + \lambda_\ell\Psi_\ell(r)}\dd\tilde\rho(r)\,.
\end{align*}
The remaining integral can be easily evaluated using numerical integration, as the density of $\tilde\rho$ is available in closed-form.

\subsection{Proof of Theorem~\ref{thm:sgd_radvi_converges}}\label{app:sgd_radvi_converges_proof}
To prove the main result, we follow the strategy of \citet[Lemma 5.16]{JiaChePoo24MFVI}. 
Setting up notation, recall that\looseness-1
\begin{align*}
    & {\nabla}_\lambda\cV((T_\lambda)_\sharp\rho) = \nabla_\lambda \E_{X\sim\rho}[V(T_\lambda(X))] = \E_{X \sim \rho}\vec\Psi(\|X\|)\langle X/\|X\|, \nabla V(T_\lambda(X))\rangle \\
    & \hat {\nabla}_\lambda\cV((T_\lambda)_\sharp\rho) = \vec\Psi(\|\hat X\|)\langle \hat X/\|\hat X\|, \nabla V(T_\lambda(\hat X))\rangle\,,
\end{align*}
where $\hat X \sim \rho$. As the gradient is only stochastic in the potential term, we need to bound
\begin{align*}
    \E\|Q^{-1}(\hat{\nabla}_\lambda\cF(T_\lambda) - {\nabla}_\lambda\cF(T_\lambda)\|^2_{Q} &= \E\|Q^{-1/2}(\hat{\nabla}_\lambda\cV((T_\lambda)_\sharp\rho) - {\nabla}_\lambda\cV((T_\lambda)_\sharp\rho)\|^2 \leq c_0 + c_1 \E\|T_\lambda - T_{\lambda^\star}\|^2_{L^2(\rho)}\,,
\end{align*}
for constants $c_0,c_1 > 0$ which depend on the problem parameters (e.g., $\ell_V,L_V,d$).

To start, note that it suffices to use the following bound
\begin{align*}
   \tr {\rm Cov}(Q^{-1/2}\vec\Psi(\|X\|)\langle X/\|X\|, \nabla V(T_\lambda(X))\rangle) &= \E \langle Q^{-1}, \vec\Psi(\|X\|)\vec\Psi(\|X\|)^{\top}\rangle \langle X/\|X\|, \nabla V(T_\lambda(X))\rangle^2) \\
   &\leq J^3 \E_{X\sim\rho}\|\nabla V \circ T_\lambda(X)\|^2\,.
\end{align*}
where we bound $\langle Q^{-1},\vec\Psi(\|x\|)\vec\Psi(\|x\|)^\top\rangle \lesssim J^3$. This follows by mimicking Lemma 5.15 by \citet{JiaChePoo24MFVI} but using the computations from the proof of Proposition~\ref{prop:cf_smooth_convex}; we omit this computation. We bound this last term as follows
\begin{align*}
    \E_\rho\|\nabla V \circ T_\lambda\|^2
    &\leq 2\,\E_\rho\|\nabla V\circ T_\lambda - \nabla V\circ T_{\lambda^\star}\|^2 + 2\, \E_\rho\|\nabla V\circ T_{\lambda^\star}\|^2 \\
    &\leq 2 L^2_V\, \|T_\lambda - T_{\lambda^\star}\|^2_{L^2(\rho)} + 2\, \E_\rho\|\nabla V\circ T_{\lambda^\star}\|^2 \\
    &\leq 2 L^2_V\, \|T_\lambda - T_{\lambda^\star}\|^2_{L^2(\rho)} + 4L^2_V\, \|T_{\lambda^\star} - T^\star_{\rm rad}\|^2_{L^2(\rho)}\, + 4\E_\rho\|\nabla V\circ T^\star_{\rm rad}\|^2 \\
    &\leq 2 L^2_V\, \|T_\lambda - T_{\lambda^\star}\|^2_{L^2(\rho)} + 4L^2_V\, \|T_{\lambda^\star} - T^\star_{\rm rad}\|^2_{L^2(\rho)}\, + 4\kappa^2 L_V d\,,
\end{align*}
where the last inequality follows from Lemma~\ref{lem:sgd_lemma1}. Using a crude upper bound of $L_V\|T_{\lambda^\star} - T^\star_{\rm rad}\|^2_{L^2(\rho)} \leq \kappa d$ (which can be derived via Theorem~\ref{theorem_unif_approx}, we can simplify the bound to
\begin{align*}
\E\|Q^{-1}(\hat{\nabla}_\lambda\cF(T_\lambda) - {\nabla}_\lambda\cF(T_\lambda)\|^2_{Q} &\lesssim J^3 L_V^2 \E\|T_\lambda - T_{\lambda^\star}\|^2_{L^2(\rho)} + J^3 \Bigl(L_V^2\|T_{\lambda^\star} - T_{\rm rad}^\star\|^2_{L^2(\rho)} + \kappa^2L_Vd\Bigr) \\
&\lesssim J^3 L_V^2 \E\|T_\lambda - T_{\lambda^\star}\|^2_{L^2(\rho)} + J^3 \kappa^2 d L_V\,.
\end{align*}
The statement follows by employing Theorem 4.3 by \cite{JiaChePoo24MFVI}.

\begin{lemma}\label{lem:sgd_lemma1}
Let $\pi\propto \exp(-V)$ satisfy \ref{well_cond} with $V$ minimized at the origin, and let $\pirad^\star$ be the optimal radial approximation. Then  
\begin{align*}
    \E_{Y \sim \pi^\star_{\rm rad}}\|\nabla V(Y)\|^2 \leq \kappa^2 L_V d\,.
\end{align*}
\end{lemma}
\begin{proof}
    By smoothness of $V$ and strong convexity of $-\log(\pirad^\star)$
    \begin{align*}
        \E_{Y \sim \pi^\star_{\rm rad}}\|\nabla V(Y)\|^2 \leq L_V^2 \E_{Y \sim \pi^\star_{\rm rad}}\|Y\|^2 \leq \kappa^2 \E_{Y \sim \pirad^\star}\|\nabla(-\log\pirad^\star)(Y)\|^2\,.
    \end{align*}
We conclude by invoking a standard fact about log-smooth measures \citep[``basic lemma'']{ChewiBook}. 
\end{proof}

\begin{lemma}\label{lem:sgd_lemma2}
Let $\{\Psi_j\}_{j=0}^J$ be our dictionary of choice, and let $Q_{ij} = \E_{X\sim\rho}[\Psi_i(\|X\|)\Psi_j(\|X\|)]$. Then for any $x \in \R^d$, it holds that
\begin{align*}
    \langle Q^{-1}, \vec\Psi(x)\vec\Psi(x)^\top\rangle \lesssim M^3.
\end{align*}
where $M = J+1$.
\end{lemma}
\begin{proof}
Our goal is to show that ${Q}$ is lower bounded by some absolute constant $c_Q > 0$. This is sufficient, as
\begin{align*}
    \langle Q^{-1}, \vec\Psi(x)\vec\Psi(x)^\top \rangle  {\le} \frac{1}{c_Q}\tr(\vec\Psi(x)\vec\Psi(x)^\top) = \frac{1}{c_Q}\sum_{j=0}^J\Psi_j^2(\|x\|) \leq \frac{M}{c_Q}\,,
\end{align*}
where we used the trivial upper bound $\Psi_j \leq 1$. Ultimately, we want to show that $1/c_Q \lesssim M^2$. Taking $\lambda \in \R^M_+$, this is equivalent to proving the following lower bound
\begin{align*}
    \lambda^\top Q \lambda = \int \Bigl(\sum_{j=0}^J \lambda_j\Psi_j(r)\Bigr)^2 \dd \tilde\rho(r) \eqqcolon \int (\tilde T_\lambda(r))^2 \dd \tilde\rho(r) \overset{!}{\gtrsim} M^{-2}\|\lambda\|^2\,.
\end{align*}
We can lower-bound this quantity term by term along the partitions $I_0 = [0,\sqrt{d}-R]$ and $I_\ell = [a_\ell,a_\ell + \delta]$. For $\ell \geq 1$,
\begin{align*}
    \int_{a_\ell}^{a_\ell + \delta} (\tilde T_\lambda(r))^2 \dd \tilde\rho(r) \geq \lambda_\ell^2 \inf_m \int_{a_\ell}^{a_\ell + \delta} ((r - a_\ell)/\delta - m)^2 \dd \tilde\rho(r) \gtrsim \lambda_\ell^2 \Upsilon^{-1}\int_{a_\ell}^{a_\ell + \delta}\dd\tilde\rho(r)\gtrsim \lambda_\ell^2 \Upsilon^{-1}\,,
\end{align*}
where, over the interval, $\tilde\rho$ is effectively constant (recall the arguments in Proposition~\ref{prop:cf_smooth_convex}). Also, for $\ell=0$, we can use the arguments from Proposition~\ref{prop:cf_smooth_convex}) in this special case to prove that
\begin{align*}
    \int_{0}^{\delta_0} (\tilde T_\lambda(r))^2 \dd \tilde\rho(r) \geq \lambda_0^2 \delta_0^{-2} \int_0^{\delta_0} r^2 \dd \tilde\rho(r) \gtrsim \lambda_0^2 \,.
\end{align*}
Adding up all the terms, this concludes the proof as $\Upsilon \asymp M^2$.
\end{proof}

\section{Experimental details} \label{app:experiments}
\subsection{Hyperparameters}
For \texttt{radVI}, our stochastic estimates use $n=100$ samples. For the isotropic distributions: we use 10000 iterations for all experiments and the learning rate for Gaussian and Student-$t$ was set to $7\times 10^{-3}$, for the the logistic distribution, we used $5\times 10^{-2}$, and $5\times 10^{-3}$ for Laplace. In the anisotropic case, we use $7 \times 10^{-3}$ for all distributions and with 30000 iterations.

To obtain a Laplace approximation of a posterior, we use the ``minimize'' function in Scipy \citep{virtanen2020scipy} with the BFGS optimizer. Our implementation of Forward-Backward VI is exactly as in the publicly available repository by \citet{Diao2023}, and we use the same learning rate and number of iterations as our approach.

\subsection{Synthetic distributions}\label{app:analytical_ot_map}
We now go over the various synthetic distributions we considered throughout this work, omitting the trivial Gaussian case. We begin by defining the Mahalanobis distance 
$r(x) \defeq \sqrt{(x-\mu)^\top \Sigma^{-1}(x-\mu)}$, where $\mu$ is the mean and $\Sigma$ the covariance matrix.

\subsubsection{Student-\texorpdfstring{$t$}{t} distribution}
Let $\pi_{\rm Stu}$ be the Student-$t$ distribution with $\nu > 0$ degrees of freedom in $\R^d$, mean $\mu \in \R^d$, and {scale matrix} $\Sigma \in \R^{d\times d}$. {We say $X\sim \pi_{\rm Stu}$ if $X =\mu + Z/\sqrt{W/\nu}$ where $Z \sim \mathcal{N}(0,\Sigma)$ and $W \sim \chi^2_\nu$ are
independent.}
The density is given by
\begin{equation*}
\pi_{\rm Stu}(x) =
\frac{\Gamma(\frac{\nu + d}{2})}
     {\Gamma(\frac{\nu}{2})\, (\nu \pi)^{d/2}\, |\Sigma|^{1/2}}\,
\Bigl(1 + \frac{r(x)^2}{\nu}\Bigr)^{-\frac{\nu + d}{2}}\,.
\end{equation*}
Writing $\pi_{\rm Stu} \propto\exp(-V_{\rm Stu})$, the potential function $V$ is given by
\begin{align*}
V_{\rm Stu}(x) = \frac{\nu + d}{2} \log\Bigl(1 + \frac{r(x)^2}{\nu}\Bigr)
+ \frac{d}{2}\log(\nu \pi) 
+ \frac{1}{2}\log|\Sigma| + \log\Gamma\Bigl(\frac{\nu}{2}\Bigr) 
- \log\Gamma\Bigl(\frac{\nu+d}{2}\Bigr)\,.
\end{align*}

\paragraph{Closed-form radial transport map.}
Here we briefly derive the closed-form expression for the optimal transport map from the standard Gaussian to the Student-$t$ distribution, written $T^\star(x) = \Psi^\star(\|x\|)\,\frac{x}{\|x\|}$. 
To compute $T^\star$, we invoke the principle of conservation of mass: under a transport $T$ that pushes $\rho$ to $\pi$, probability mass is preserved. Formally, for all Borel sets $A\subset\mathbb{R}^d$,
\[
\pi(A) = ((T^\star)_\sharp \rho)(A) = \rho\!\left((T^\star)^{-1}(A)\right).
\]
In the radial setting, it suffices to enforce this on balls of radius $r$ i.e., on sets of the form $\{x\in\R^d :\|x\|\le r\}$, yielding, for a random variable $X\sim\rho$ and $Y\sim\pi,$
\[
\mathbb{P}(\|X\|\le r)=\mathbb{P}(\|Y\|\le \Psi^\star(r))\,.
\]
For $\rho=\mathcal{N}(0,I_d)$ we have $\|X\|^2\sim\chi_d^2$, hence
\[
\mathbb{P}(\|X\|\le r)=\mathbb{P}(\|X\|^2\le r^2)=F_{\chi^2_d}(r^2)\, ,
\]
and ${\Psi^\star(r)}$ is determined implicitly by solving
\begin{align}\label{eq:cons_mass_eq}
    F_{\chi^2_d}(r^2)=\mathbb{P}(\|Y\|\le \Psi^\star(r))\,.
\end{align}
If $Y\sim\pi_{\rm Stu}$, then $\|Y\|^2/d\sim F_{d,\nu}$, where $F_{d,\nu}$ is the CDF of the $F$-distribution with $(d,\nu)$ degrees of freedom. Hence,
\begin{align*}
\mathbb{P}(\|Y\|\le s)=\mathbb{P}\Bigl(\frac{\|Y\|^2}{d}\le \frac{s^2}{d}\Bigr)=F_{d,\nu}\Bigl(\frac{s^2}{d}\Bigr)\,.
\end{align*}
By \eqref{eq:cons_mass_eq},
\begin{align*}
F_{\chi^2_d}(r^2) = F_{d,\nu}\Bigl(\frac{\Psi^\star(r)^2}{d}\Bigr)\,,
\end{align*}
and thus,
\begin{align*}
    \Psi^\star(r) = \sqrt{d\,F^{-1}_{d,\nu}\bigl(F_{\chi^2_d}(r^2)\bigr)}\,.
\end{align*}

\subsubsection{Laplace distribution}
Let $\pi_{\rm Lap}$ be the Laplace distribution in $\R^d$. {To draw $X \sim \pi_{\rm Lap}$ with mean parameter $\mu$ and covariance parameter $\Sigma$, apply the transformation $X = \mu + \sqrt{Y}\,\Sigma^{1/2} Z$, where $Y \sim \mathrm{Exp}(1)$ 
$Z \sim \mathcal{N}(0,I_d)$ are independent. }

The density of the $d$-dimensional symmetric multivariate Laplace distribution 
with mean $\mu \in \mathbb{R}^d$ and covariance matrix $\Sigma \in \mathbb{R}^{d \times d}$ is given by
\begin{equation}
\pi_{\rm Lap}(x) = \frac{2}{ (2\pi)^{d/2} |\Sigma|^{1/2} }\left( \frac{r(x)^2}{2} \right)^{\nu/2}
K_\nu\left( \sqrt{2}\, r(x) \right),
\label{eq:laplace_density}
\end{equation}
where $\nu=(2-d)/2$, and $K_\nu$ denotes the modified Bessel function of the second kind. Writing $\pi_{\rm Lap} \propto \exp(-V_{\rm Lap})$, we have
\begin{equation}
V_{\rm Lap}(x) = - \log\Bigl(
\frac{2}{ (2\pi)^{d/2}  |\Sigma|^{1/2} }
\Bigr)
- \frac{\nu}{2} \log \Bigl( \frac{r(x)^2}{2} \Bigr)
- \log K_\nu( \sqrt{2} r(x)).
\label{eq:laplace_potential}
\end{equation}

\paragraph{Radial transport map.} For $Y \sim \pi_{\rm Lap}$, we have that
\begin{align*}
    \mathbb{P}(\|Y\|\le s)
=\frac{\displaystyle\int_0^{s}\phi(u)\dd u}
       {\displaystyle\int_0^{\infty}\phi(t)\dd t}\,, \quad  \phi(s)=s^{\frac d2}\,K_{1-\frac d2}(\sqrt{2}\,s).
\end{align*}
By \eqref{eq:cons_mass_eq}, we can determine $\Psi^\star$ implicitly by solving
\[
F_{\chi^2_d}(r^2)
= \frac{\displaystyle\int_0^{\Psi^\star(r)} u^{\frac d2} K_{1-\frac d2}(\sqrt{2}\,u)\,\dd u}{\displaystyle\int_0^{\infty} t^{\frac d2} K_{1-\frac d2}(\sqrt{2}\,t)\,\dd t}.
\]

Solving the implicit equation for the mapped radius $\Psi^{\star}(r)$ point-wise is computationally prohibitive. To circumvent this, we use an interpolation-based approach: the radial CDF of the target is pre-computed on a dense grid, and a linear interpolator of its inverse is constructed. The map then reduces to a direct, vectorized evaluation of this interpolator. Alternatively, the equation can be solved for each point using a root-finding algorithm, such as Brent's method \citep{brent1973some}, although this approach is less efficient for large datasets.

\subsubsection{Logistic distribution}
Let $\pi_{\rm Log}$ be the logistic distribution in $\R^d$ with mean $\mu$, covariance $\Sigma$, and scale parameter $s > 0$. {Unlike the other cases, we use rejection sampling here. The target distribution is elliptical, so it suffices to sample the Mahalanobis radius, whose density is
\[
f_R(r) \propto r^{d-1}\,\frac{\exp(-r/s)}{(1+\exp(-r/s))^{2}}\,, \qquad r>0\,.
\]
As proposal we take $R \sim \mathrm{Gamma}(d,s)$, which has density 
$g(r) \propto r^{d-1}\exp(-r/s)$.  
The acceptance probability is
\[
\alpha(r) = \frac{f_R(r)}{g(r)} = \frac{1}{(1+\exp(-r/s))^{2}}\,.
\]
Accepted radii are combined with a uniform direction $U$ on the unit sphere, 
yielding $X = \mu + \Sigma^{1/2}(R U)$.}

The density of the $d$-dimensional multivariate logistic distribution with mean 
$\mu \in \mathbb{R}^d$, covariance matrix $\Sigma \in \mathbb{R}^{d \times d}$, 
and scale parameter $s > 0$ is
\begin{equation}\label{eq:logistic_density}
\pi_{\rm Log}(x) = \frac{1}{Z(d,\Sigma,s)}\,
\frac{\exp(-r(x)/s)}{(1+\exp(-r(x)/s))^{2}}\,,
\end{equation}
where
\[
Z(d,\Sigma,s) =
|\Sigma|^{1/2}\,
\frac{2 \pi^{d/2}}{\Gamma(d/2)}\,
s^{d}\,
\Gamma(d)\,
(1-2^{{-(d-2)}})\,
\zeta({d-1})\,.
\]
Similarly as before, taking the negative logarithm of \eqref{eq:logistic_density} realizes the potential $V$,
\begin{equation}
V_{\rm Log}(x) =
\frac{r(x)}{s}
+ 2\log(1+\exp(-r(x)/s))
+ \log Z(d,\Sigma,s).
\label{eq:logistic_potential}
\end{equation}

\paragraph{Radial transport map.} For $Y \sim \pi_{\rm Log}$, it holds that
\begin{align*}
    \mathbb{P}(\|Y\|\le s)
=\frac{\displaystyle\int_0^{s}\phi(u)\dd u}
       {\displaystyle\int_0^{\infty}\phi(t) \dd t}\,, \quad \phi(s)=s^{d-1}\,\frac{e^{-s}}{(1+e^{-s})^2}\,.
\end{align*}
By \eqref{eq:cons_mass_eq}, we can determine $\Psi^\star$ implicitly by solving
\[
F_{\chi^2_d}(r^2)
=\frac{\displaystyle\int_0^{\Psi^\star(r)} u^{d-1}\,\frac{e^{-u}}{(1+e^{-u})^2} \dd u}
       {\displaystyle\int_0^{\infty} t^{d-1}\,\frac{e^{-t}}{(1+e^{-t})^2} \dd t}\,.
\]
in the same way as in the case of the Laplace distribution.

\section{Auxiliary computational results}\label{sec:aux_figures}

\subsection{Higher dimensions}\label{sec:high_dim_examples}

Here, we compare methods for isotropic distributions as in Table~\ref{tab:table1}, but now with $d=100$. We set the meshsize to $d^{-1/8}$ and double the number of iterations. All other parameters are left the same, and the results are reported below in Table~\ref{tab:table2}.
\begin{table}[h]
\centering
\resizebox{0.65\linewidth}{!}{%
\begin{tabular}{lcccc}
& \multicolumn{4}{c}{Isotropic targets} \\
\cmidrule(lr){2-5}
Method & Gaussian &  Laplace & Logistic  & Student-$t$ \\
\toprule
LA              &   $1.48 \times 10^{-4}$    & $43.04$   &  $7450$  &  $68.89$ \\
GVI             &    $5.07 \times 10^{-4}$  &   $18.34$   &  $8.67$   & $5.21$ \\
\texttt{radVI}  &     $3.71 \times 10^{-4}$  &  $7.67 \times 10^{-2}$    & $1.96 \times 10^{-1}$   & $1.89\times 10^{-1}$ \\
\bottomrule
\end{tabular}}
\caption{Estimated Wasserstein distance between various VI solutions for learning isotropic targets in $d=100$.}\label{tab:table2}
\end{table}

\subsection{Figures}
\begin{figure}[!h]
    \centering
    \includegraphics[width=0.5\linewidth]{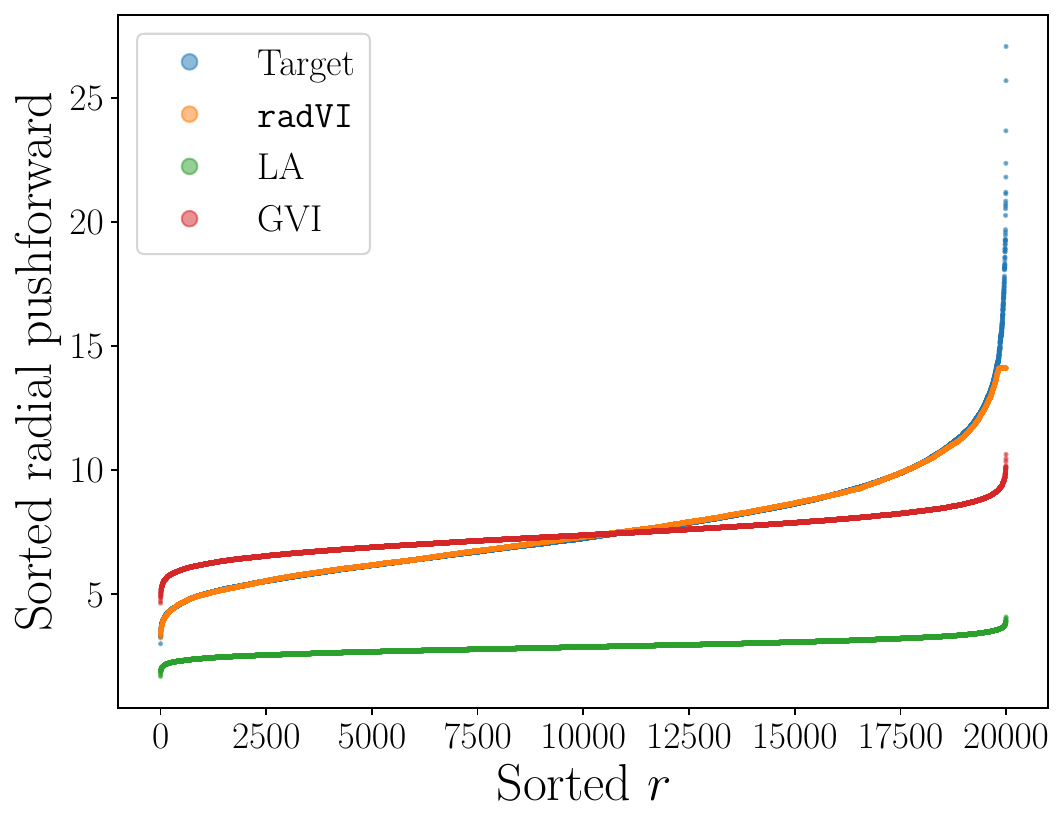}
    \caption{Comparing learned radial profiles of \texttt{radVI} versus other approximation methods for the isotropic Student-$t$ distribution.}
    \label{fig:isotropic_student}
\end{figure}

\begin{figure}[!h]
    \centering
    \includegraphics[width=0.5\linewidth]{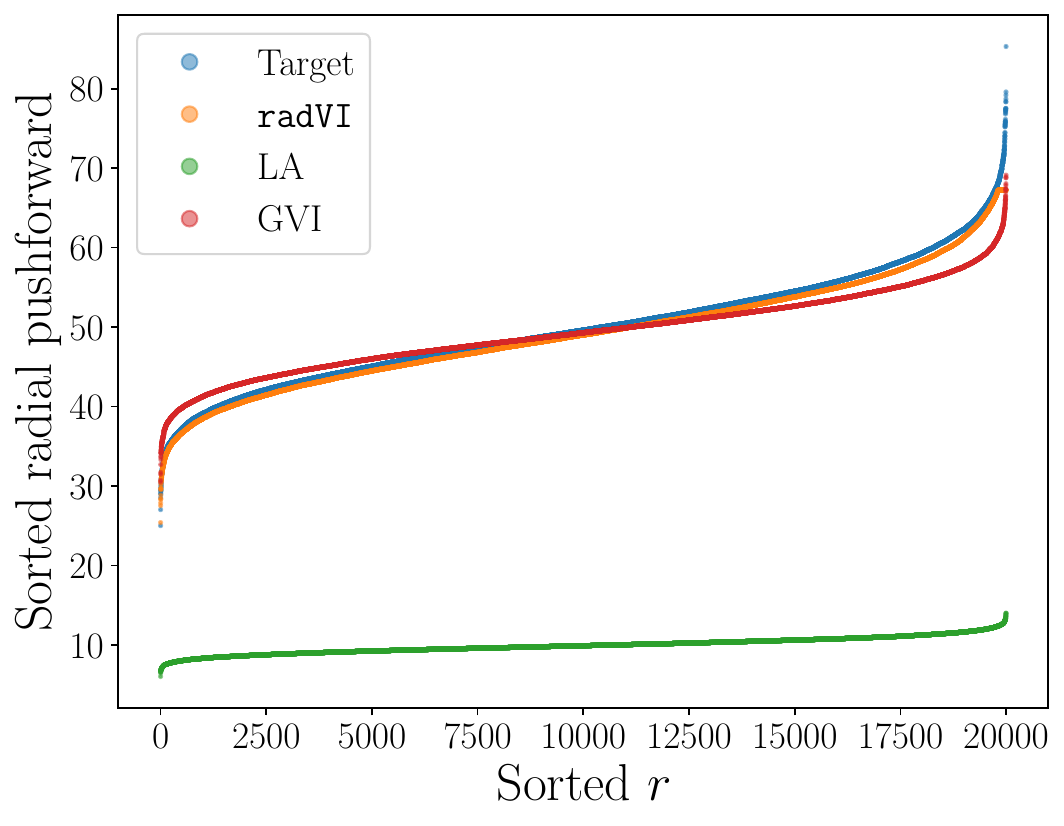}
    \caption{Comparing learned radial profiles of \texttt{radVI} versus other approximation methods for the isotropic logistic distribution.}
\label{fig:isotropic_logistic}
\end{figure}

\begin{figure}
    \centering
\includegraphics[width=0.5\linewidth]{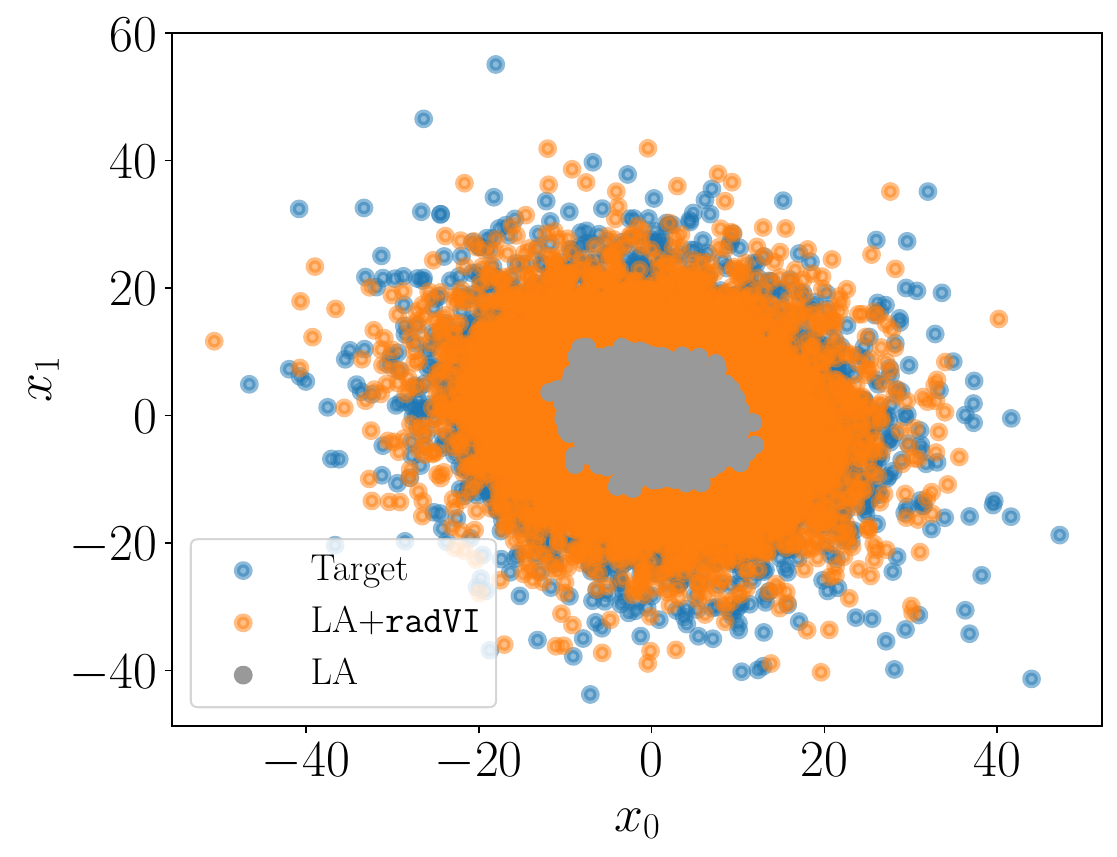}
    \caption{Visual comparison of true target samples, those generated by LA, and ours (LA+\texttt{radVI}), for learning the anisotropic Student-$t$ distribution.}
    \label{fig:anisotropic_student}
\end{figure}

\begin{figure}[!t]
\centering
\hspace{-0.1cm}
\includegraphics[width=0.44\textwidth]{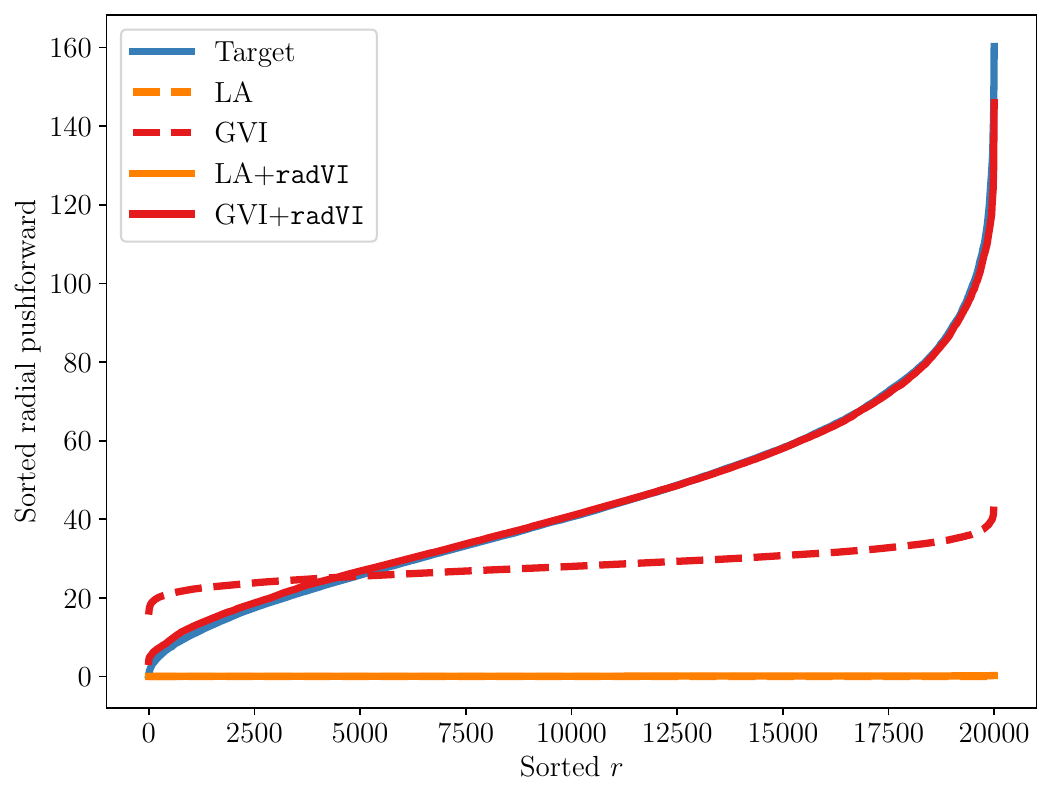}
\includegraphics[width=0.44\textwidth]{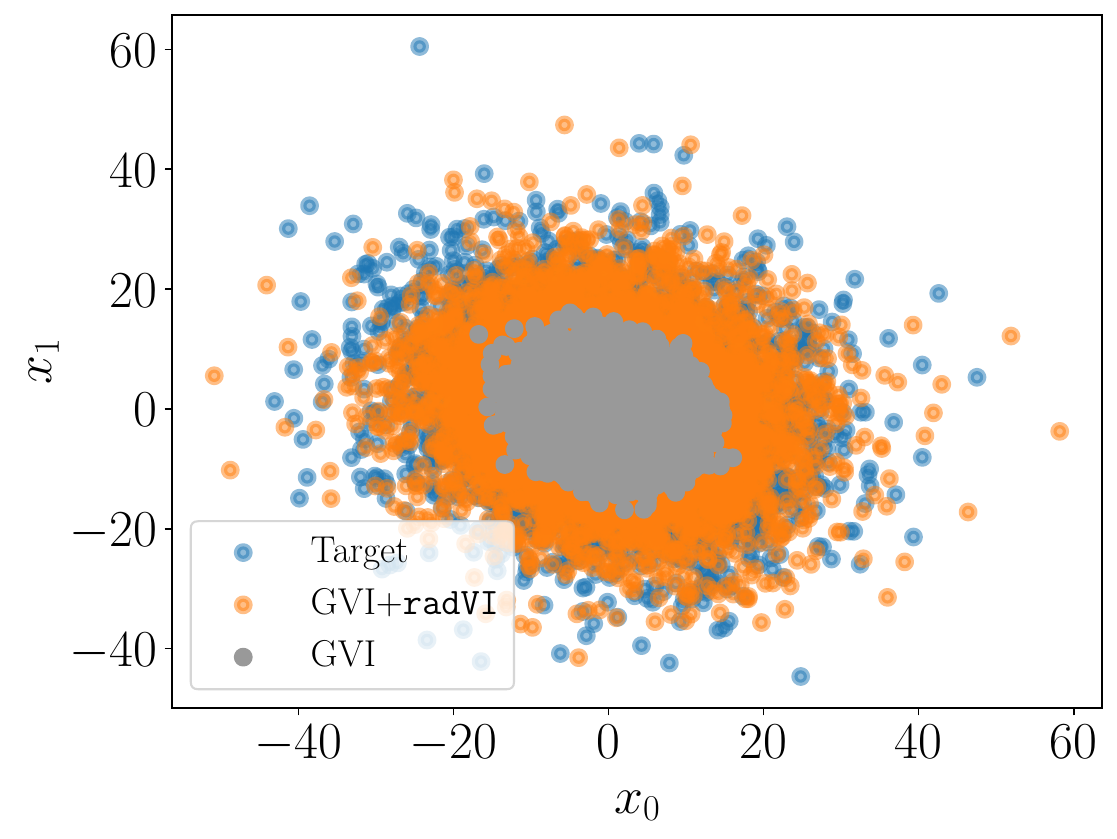}
\vspace{-0.4cm}
\caption{\textbf{Left:} Comparing whitening methods for learning the anisotropic Laplace distribution, with and without \texttt{radVI}. \textbf{Right:} Visual comparison of true target samples, those generated by GVI, and ours (GVI+\texttt{radVI}). \label{fig:anisotropic_laplace}}
\end{figure}

\end{document}